\documentclass[accepted]{uai2021}
\usepackage[american]{babel}

\usepackage[
  style=authortitle,
  citestyle=authoryear,
  maxnames=2,
  maxbibnames=99,
  natbib=true,
  hyperref=true,
  backend=bibtex]{biblatex}

\DeclareNameAlias{sortname}{last-first}

\renewbibmacro{in:}{}

\usepackage{tikz} 
\usetikzlibrary{calc,decorations.pathmorphing,shapes,arrows}

\newcommand\Lsquigarrow[1]{%
\mathrel{%
\begin{tikzpicture}[baseline= {( $ (current bounding box.south) + (0,-0.5ex) $ )}]
  \node[inner sep=.5ex] (a) {$\scriptstyle #1$};
  \path[draw,implies-,double distance between line centers=1.5pt,decorate,
    decoration={zigzag,amplitude=0.7pt,segment length=1.2mm,pre=lineto,
    pre   length=4pt}] 
    (a.south west) -- (a.south east);
\end{tikzpicture}}%
}

\DeclareCiteCommand{\cite}
  {\usebibmacro{prenote}}
  {\usebibmacro{citeindex}%
   \printtext[bibhyperref]{\usebibmacro{cite}}}
  {\multicitedelim}
  {\usebibmacro{postnote}}

\DeclareCiteCommand*{\cite}
  {\usebibmacro{prenote}}
  {\usebibmacro{citeindex}%
   \printtext[bibhyperref]{\usebibmacro{citeyear}}}
  {\multicitedelim}
  {\usebibmacro{postnote}}

\DeclareCiteCommand{\parencite}[\mkbibparens]
  {\usebibmacro{prenote}}
  {\usebibmacro{citeindex}%
    \printtext[bibhyperref]{\usebibmacro{cite}}}
  {\multicitedelim}
  {\usebibmacro{postnote}}

\DeclareCiteCommand*{\parencite}[\mkbibparens]
  {\usebibmacro{prenote}}
  {\usebibmacro{citeindex}%
    \printtext[bibhyperref]{\usebibmacro{citeyear}}}
  {\multicitedelim}
  {\usebibmacro{postnote}}

\DeclareCiteCommand{\footcite}[\mkbibfootnote]
  {\usebibmacro{prenote}}
  {\usebibmacro{citeindex}%
  \printtext[bibhyperref]{ \usebibmacro{cite}}}
  {\multicitedelim}
  {\usebibmacro{postnote}}

\DeclareCiteCommand{\footcitetext}[\mkbibfootnotetext]
  {\usebibmacro{prenote}}
  {\usebibmacro{citeindex}%
   \printtext[bibhyperref]{\usebibmacro{cite}}}
  {\multicitedelim}
  {\usebibmacro{postnote}}

\DeclareCiteCommand{\textcite}
  {\boolfalse{cbx:parens}}
  {\usebibmacro{citeindex}%
   \printtext[bibhyperref]{\usebibmacro{textcite}}}
  {\ifbool{cbx:parens}
     {\bibcloseparen\global\boolfalse{cbx:parens}}
     {}%
   \multicitedelim}
  {\usebibmacro{textcite:postnote}}

\usepackage{mathtools} 
\usepackage{ebproof}
\usepackage{microtype}
\usepackage{graphicx}
\usepackage{booktabs} 
\usepackage{hyperref}

\usepackage{amssymb,amsthm}
\usepackage{latexsym}
\usepackage{bm}
\usepackage{stmaryrd}
\usepackage{cancel}
\usepackage{mathtools}
\usepackage{algpseudocode}
\usepackage{algorithm}
\usepackage{subcaption}
\usepackage[scale=0.925]{inconsolata}

\newcommand{\citetApp}[1]{\citet{#1}}
\newcommand{\citepApp}[1]{\citep{#1}}
\newcommand{\citeApp}[1]{\citet{#1}}
\usepackage{listings}
\usepackage{color}
\usepackage{xcolor}
\definecolor{blue}{rgb}{0,0.2,0.5}
\definecolor{green}{rgb}{0.1,0.35,0.0}
\definecolor{red}{rgb}{0.5,0.0,0.0}
\definecolor{purple}{rgb}{0.4,0,0.6}
\definecolor{cyan}{rgb}{0.0,0.4,0.3}
\definecolor{orange}{rgb}{0.6,0.4,0.0}
\definecolor{gray}{rgb}{0.3,0.3,0.3}

\lstdefinestyle{probtorch}%
{language=Python,
 sensitive, %
 basicstyle=\ttfamily,
 keywordstyle=\color{purple},
 emphstyle=\color{red},
 commentstyle=\em\color{gray},
 stringstyle=\color{orange},
 keywordstyle=[1]\color{purple},
 keywordstyle=[2]\color{cyan},
 keywordstyle=[3]\color{blue},
 keywordstyle=[4]\color{red},
 keywordstyle=[5]\color{purple},
 morekeywords=[1]{},
 morekeywords=[2]{Multinomial, State, Normal, ones, inc_kl, step, zero_grad, backward, next_batch, partial}, %
 morekeywords=[3]{factor, sample, observe}, %
 morekeywords=[4]{propose, compose, extend, resample, condition, batch}, 
 morekeywords=[5]{}
}[keywords,comments,strings]

\lstnewenvironment{probtorch}[1][]{%
    \let\lstoldname\lstlistingname
    \renewcommand{\lstlistingname}{Program}
    \lstset{
        style=probtorch, 
        basicstyle=\ttfamily\small,
        mathescape=true,
        captionpos=b, 
        #1}
}{%
  \renewcommand{\lstlistingname}{\lstoldname}
}

\newcommand{\pt}{\lstinline[mathescape,keepspaces,style=probtorch]}
\newcommand{\mpt}[1]{\mbox{\pt!#1!}}

\hypersetup{%
  colorlinks,
  linkcolor=red,
  citecolor=blue, 
  urlcolor=blue   
}

\newtheorem{definition}{Definition}
\newtheorem{proposition}{Proposition}
\newtheorem{theorem}{Theorem}
\newtheorem{lemma}{Lemma}

\newcommand{\eval}{\,\reflectbox{$\leadsto$}\:}
\newcommand{\ceval}[2]{\eval #2[#1]}

\newcommand{\den}[1]{\llbracket #1 \rrbracket}
\renewcommand{\v}[1]{\ensuremath{\vec{#1}}}
\makeatletter
\def\blfootnote{\xdef\@thefnmark{}\@footnotetext}
\makeatother

\title{Learning Proposals for Probabilistic Programs with Inference Combinators}
\author[1]{Sam Stites\textsuperscript{*}} 
\author[1]{Heiko Zimmermann\textsuperscript{*}}
\author[1]{Hao Wu}
\author[1]{Eli Sennesh}
\author[1]{Jan-Willem van de Meent}
\affil[1]{%
    Khoury College of Computer Sciences\\
    Northeastern University\\
    Boston, Massachusetts, USA
}
\affil[1]{%
    \texttt{$\{$stites.s, zimmermann.h, wu.hao10, e.sennesh, j.vandemeent$\}$@northeastern.edu}
}

\begin{document}
\maketitle

\blfootnote{$^*$ Equal contribution}
%

\begin{abstract}%
We develop operators for construction of proposals in probabilistic programs, which we refer to as inference combinators. 
Inference combinators define a grammar over importance samplers that compose primitive operations such as application of a transition kernel and importance resampling. Proposals in these samplers can be parameterized using neural networks, which in turn can be trained by optimizing variational objectives. The result is a framework for user-programmable variational methods that are correct by construction and can be tailored to specific models. 
We demonstrate the flexibility of this framework by implementing advanced variational methods based on amortized Gibbs sampling and annealing.
\end{abstract}

\begin{refsection}
\vspace{-0.5\baselineskip}
\section{Introduction}
\vspace{-0.25\baselineskip}
One of the major ongoing developments in probabilistic programming is the integration of deep learning with approaches for inference. Libraries such as Edward \citep{tran2016edward}, Pyro \citep{bingham2018pyro}, and Probabilistic Torch \citep{siddharth2017learning}, extend deep learning frameworks with functionality for the design and training of deep probabilistic models. Inference in these models is typically performed with amortized variational methods, which learn an approximation to the posterior distribution of the model. 

Amortized variational inference is widely used in the training of variational autoencoders (VAEs). While standard VAEs employ an unstructured prior over latent variables, such as a spherical Gaussian, in probabilistic programming, we are often interested in training a neural approximation to a simulation-based model \citep{baydin-2019-etalumis}, or more generally to perform inference in models that employ programmatic priors as inductive biases. This provides a path to incorporating domain knowledge into methods for learning structured representations in a range of applications, such as identifying objects in an image or a video \citep{greff2019multi}, representing users and items in reviews \citep{esmaeili2019structured}, and characterizing the properties of small molecules \citep{kusner2017grammar}. 

While amortized variational methods are very general, they often remain difficult to apply to structured domains, 
where gradient estimates based on a small number of (reparameterized) samples are often not sufficient. A large body of work improves upon standard methods by combining variational inference with more sophisticated sampling schemes based on Markov chain Monte Carlo (MCMC) and importance sampling (e.g.~\citet{naesseth2018variational,le2019revisiting,wu2020amortized}; see Appendix~\ref{apx:related} for a comprehensive discussion). However to date, few of these methods have been implemented in probabilistic programming systems. One reason for this is that many methods rely on a degree of knowledge about the structure of the underlying model, which makes it difficult to develop generic implementations. 

In this paper, we address this difficulty by considering a design that goes one step beyond that of traditional probabilistic programming systems. Instead of providing a language for the definition of models in the form of programs, we introduce a language for the definition of sampling strategies that are applicable to programs. 
This is an instance of a general idea that is sometimes referred to as \emph{inference programming} \citep{mansinghka2014venture}. Instantiations of this idea include inference implementations as monad transformers \citep{Scibior:2018:FPM:3243631.3236778}, inference implementations using low-level primitives for integration \citep{obermeyer2019functional}, and programming interfaces based on primitive operations for simulation, generation, and updating of program traces \citep{cusumano-towner2019gen}.

Here we develop a new approach to inference programming based on constructs that we refer to as \emph{inference combinators}, which act on probabilistic programs that perform importance sampling during evaluation. A combinator accepts one or more programs as inputs and returns a new program. Each combinator denotes an elementary operation in importance sampling, such as associating a proposal with a target, composition of a program and a transition kernel, and importance resampling. 
The result is a set of constructs that can be composed to define user-programmable importance samplers for probabilistic programs that are valid by construction, in the sense that they generate samples that are properly weighted \citep{liu2008monte, naesseth2015nested} under the associated target density of the program. These samplers in turn form the basis for nested variational methods \citep{zimmermann2021nested} that minimize a KL divergence to learn neural proposals.

We summarize the contributions of this paper as follows: 
\begin{itemize}
    \item We develop a language for inference programming in probabilistic programs, in which combinators define a grammar over properly-weighted importance samplers that can be tailored to specific models.
    \item We formalize semantics for inference and prove that samplers in our language are properly weighted for the density of a program.
    \item We demonstrate how variational methods can be used to learn proposals for these importance samplers, and thereby define a language for user-programmable stochastic variational methods.
    \item We provide an implementation of inference combinators for the Probabilistic Torch library\footnote{Code is available at \href{https://github.com/probtorch/combinators}{\texttt{github.com/probtorch/combinators}}.}. We evaluate this implementation by using it to define existing state-of-the-art methods for probabilistic programs that improve over standard methods.
\end{itemize}

\vspace{-0.5\baselineskip}
\section{Preliminaries}
\label{sec:preliminaries}
\vspace{-0.25\baselineskip}

The combinators in this paper define a language for samplers that operate on probabilistic programs. We will refer to this language as the \emph{inference language}. The probabilistic programs that inference operates upon are themselves defined in a \emph{modeling language}. 
Both languages require semantics. For the inference language, semantics formally define the sampling strategy, whereas the semantics of the modeling language define the target density for the inference problem.

We deliberately opt not to define semantics for a specific modeling language. Our implementation is based on Probabilistic Torch, but combinators are applicable to a broad class of modeling languages that can incorporate control flow, recursion, and higher-order functions. Our main technical requirement is that all sampled and observed variables must have tractable conditional densities. This requirement is satisfied in many existing languages, including Church \citep{goodman2008church}, Anglican \citep{wood2014new}, WebPPL \citep{goodman2014design}, Turing \citep{ge2018turing}, Gen \citep{cusumano-towner2019gen}, Pyro \citep{bingham2018pyro}, and Edward2 \citep{tran2016edward}\footnote{See Appendix~\ref{apx:related} for an  extensive discussion of related work.}. 

To define semantics for the inference language, we assume the existence of semantics for the modeling language. Concretely, we postulate that there exist \emph{denotational measure semantics} for the density of a program. We will also postulate that programs have corresponding \emph{operational sampler semantics} equivalent to likelihood weighting. In the next section, we will use these axiomatic semantics to define operational semantics for inference combinators such that samplers defined by these combinators are valid by construction, in the sense that evaluation yields \emph{properly weighted} samples for the density denoted by the program. Below, we discuss the preliminaries that we need for this exposition.


\vspace{-0.5\baselineskip}
\subsection{Likelihood Weighting}
\vspace{-0.5\baselineskip}
\label{subsec:likelihood_weighting_ppls}

Probabilistic programs define a distribution over variables in a programmatic manner. As a simple running example, we consider the following program in Probabilistic Torch
\begin{probtorch}[caption=A deep generative mixture model, label=prog:deepmix]
$\eta^v$, $\eta^x$ = ... # (initialize generator networks)
def f(s, x):
  # select mixture component 
  z = s.sample(Multinomial(1, 0.2*ones(5)),"z")
  # sample image embedding
  v = s.sample(Normal($\eta^v_\mu$(z), $\eta^v_\sigma$(z)),"v")
  # condition on input image
  s.observe(Normal($\eta^x$(v), 1), x,"x")
  return s, x, v, z
\end{probtorch}
This program corresponds to a density $p(x,v,z)$, in which $z$ and $v$ are unobserved variables and $x$ is an observed variable. We first define a multinomial prior $p(z)$, and then define conditional distributions $p(v \mid z)$ and $p(x \mid v)$ using a set of generator networks $\eta$. The goal of inference is to reason about the posterior $p(v,z \mid x)$. We refer to the object \pt{s} as the \emph{inference state}. This data structure stores variables that need to be tracked as side effects of the computation, which we discuss in more detail below.

In this paper, the base case for evaluation performs likelihood weighting. This is a form of importance sampling in which unobserved variables are sampled from the prior, and are assigned an (unnormalized) importance weight according to the likelihood. In the example above, we would typically execute the program in a vectorized manner; we would input a tensor of $B$ samples $x^b \sim \hat{p}(x)$ from an empirical distribution, generate tensors of $B \times L$ samples $v^{b,l}, z^{b,l} \sim p(v, z)$ from the prior, and compute 
\begin{align*}
    w^{b,l} = \frac{p(x^b, v^{b,l}, z^{b,l})}
                  {p(v^{b,l}, z^{b,l})}
            = p(x^b \mid v^{b,l}, z^{b,l}).
\end{align*}
These weights serve to compute a self-normalized approximation of a posterior expectation of some function $h(x,v,z)$ 
\begin{align}
  \begin{split}
  \label{eq:self-normalized}
  &
  \mathbb{E}_{\hat{p}(x) \: p(v,z \mid x)} \left[ 
    h(x,v,z)
  \right]
  \simeq
  \frac{1}{B}
  \sum_{b,l}
  \frac{w^{b,l}}
      {\sum_{l'} w^{b,l'}}
  h(x^b, v^{b,l}, z^{b,l}).
  \end{split}
\end{align}
Self-normalized estimators are consistent, but not unbiased. However, it is the case that each weight $w^{b,l}$ is a unbiased estimate of the marginal likelihood $p(x^b)$. 

\subsection{Traced Evaluation}
\vspace{-0.5\baselineskip}
\label{sec:traced-eval}

Probabilistic programs can equivalently denote densities and samplers. In the context of specific languages, these two views can be formalized in terms of denotational \emph{measure semantics} and operational \emph{sampler semantics} \citep{scibior2017denotational}. To define operational semantics for inference combinators, we begin by postulating sampler semantics for programs that perform likelihood weighting. We assume these semantics define an evaluation relation 
\begin{align*}
    c,\tau,\rho,w \eval \mpt{f($c^\prime$)}.
\end{align*}
In this  relation, \mpt{f} is a program, $c'$ are its input arguments, and $c$ are its return values. Evaluation additionally outputs a weight $w$, a density map $\rho$, and a trace $\tau$.\footnote{In Probabilistic Torch, we return $\tau$, $\rho$, and $w$ by storing them in the inference state \pt{s} during evaluation.}

A trace stores values for all sampled variables in an execution. We mathematically represent a trace as a mapping, 
\begin{align*}
    \tau = [\alpha_1 \mapsto c_1, \dots, \alpha_n \mapsto c_n].
\end{align*}
Here each $\alpha_i$ is an address for a random variable and $c_i$ is its corresponding value. Addresses uniquely identify a random variable. In Probabilistic Torch, as well as in languages like Pyro and Gen, each \pt{sample} or \pt{observe} call accepts an address as the identifier. In the example above, evaluation returns a trace $[\mpt{"v"} \mapsto v, \mpt{"z"} \mapsto z]$, where $z,v \sim p(z,v)$. 

The density map stores the value of the conditional density for all variables in the program, 
\begin{align*}
    \rho &= [\alpha_1 \mapsto r_1, \dots, \alpha_n \mapsto r_n], 
    &
    r_i \in [0, \infty).
\end{align*}
Whereas the trace only stores values for unobserved variables, the density map stores probability densities for both observed and unobserved variables. In the example above, evaluation of a program would output a map for the variables with addresses \pt{"x"}, \pt{"v"} and \pt{"z"}, 
\begin{align*}
  [\mpt{"x"} \mapsto p(x \mid v), \mpt{"v"} \mapsto p(v | z), \mpt{"z"} \mapsto p(z)],
\end{align*}
in which conditional densities are computed using the traced values $v = \tau(\mpt{"v"})$, $z=\tau(\mpt{"z"})$, and the observed value $x$, which is an input to the program.

We define the weight in the evaluation as the joint probability of all observed variables (i.e.~the likelihood). Since the density map $\rho$ contains entries for all variables, and the trace only contains unobserved variables, the likelihood is
\begin{align*}
    w 
    &=
    \!\!\!\!\!\!
    \prod_{\alpha \in \text{dom}(\rho) \setminus \text{dom}(\tau)} 
    \!\!\!\!\!\! 
    \rho(\alpha).
\end{align*}
To perform importance sampling, we will use a trace from one program as a proposal for another program. For this purpose, we define an evaluation under substitution
\begin{align*}
    c, \tau, \rho, w \ceval{\tau'}{\mpt{f($c^\prime$)}}.
\end{align*}
In this relation, values in  $\tau'$ are substituted for values of unobserved random variables in \pt{f}. This is to say that evaluation \emph{reuses} $\tau(\alpha)=\tau'(\alpha)$ when a value exists at address $\alpha$, and samples from the program prior when it is not. This is a common operation in probabilistic program inference, which also forms the basis for traced Metropolis-Hastings methods \citep{wingate2011lightweight}.

When performing evaluation under substitution, the set of reused variables is the intersection $\text{dom}(\tau) \cap \text{dom}(\tau')$. Conversely, the newly sampled variables are $\text{dom}(\tau) \setminus \text{dom}(\tau')$. 
We define the weight of an evaluation under substitution as the joint probability of all observed and reused variables
\begin{align}
    \label{eq:conditioned-weight}
    w &= \prod_{\alpha \in \text{dom}(\rho) \setminus \big( \text{dom}(\tau) \setminus \text{dom}(\tau') \big)} \rho(\alpha).
\end{align}

\vspace{-0.5\baselineskip}
\subsection{Denotational Semantics}
\vspace{-0.5\baselineskip}

To reason about the validity of inference approaches, we need to formalize what density a program denotes. For this purpose, we assume the existence of denotational semantics that define a prior and unnormalized density 
\begin{align*}
    \den{\mpt{f}(c')}_\gamma (\tau) = \gamma_{f}(\tau ; c'), 
    &&
    \den{\mpt{f}(c')}_p (\tau) = p_{f}(\tau ; c')
    .
\end{align*}
Given an unnormalized density, the goal of inference is to approximate the corresponding normalized density
\begin{align*}
    \pi_f(\tau ; c') 
    = 
    \frac{
        \gamma_f(\tau ; c')
    }{
        Z_f (c')
    }
    , &&
    Z_f (c') = \int d\tau \: \gamma_f(\tau ; c').
\end{align*}
The reason that we specify a program as a density over traces, rather than as a density over specific variables, is that programs in higher-order languages with control flow and/or recursion may not always instantiate the same set of variables. A program could, for example, perform a random search that instantiates different variables in every evaluation \citep{vandemeent2016black}. We refer to \citet{vandemeent2018introduction} for a more pedagogical discussion of this point. 

Formal specification of denotational semantics gives rise to substantial technical questions\footnote{Traditional measure theory based on Borel sets does not support higher-order functions. Recent work addresses this problem using a synthetic measure theory for quasi-Borel spaces, which in turn serves to formalize denotational semantics for a simply-typed lambda calculus \citep{scibior2017denotational}. In practice, these technical issues do not give rise to problems, since existing languages cannot construct objects that give rise to issues with measurability.}, but in practice the unnormalized density is computable for programs in most probabilistic languages. In the languages that we consider here, in which conditional densities for all random variables are tractable, the unnormalized density is simply the joint probability of all variables in the program. Concretely, we postulate the following relationship between the sampler and the measure semantics of a program
\[
\begin{prooftree}
  \hypo{c,\tau,\rho,w \eval \mpt{f($c^\prime$)}}
  \infer1{\displaystyle
          \den{\mpt{f($c^\prime$)}}_\gamma(\tau) 
          =
          \!\!\!\! \prod_{\alpha \in \text{dom}(\rho)} \!\!\!
          \rho(\alpha) 
          \quad
          \den{\mpt{f($c^\prime$)}}_p(\tau)
          =\!\!\! 
          \!\!\!\! \prod_{\alpha \in \text{dom}(\tau)} \!\!\!
          \rho(\alpha)}
\end{prooftree}
\]
In this notation, the top of the rule lists conditions, and the bottom states their implications. This rule states that, for any trace $\tau$ and density map $\rho$ that can be generated by evaluating \pt{f($c'$)}, the unnormalized density $\den{\mpt{f($c^\prime$)}}_\gamma(\tau)$ evaluates to the product of the conditional probabilities in $\rho$, whereas the prior density $\den{\mpt{f($c^\prime$)}}_p(\tau)$ corresponds to the product for all unobserved variables. This implies that the trace is distributed according to the prior, and that the weight is the ratio between the unnormalized density and the prior
\begin{align*}
w &= \gamma_f(\tau; c') / p_f(\tau ; c'),
&
\tau \sim p_f(\tau ; c').
\end{align*}
The above rule implicitly defines the support $\Omega_f$ of the density, in that it only defines the density for traces $\tau$ that can be generated by evaluating \pt{f($c'$)}. In the languages that we are interested in here, $\Omega_f$ may not be statically determinable through program analysis, but our exposition does not require its explicit characterization.

Finally, we define the density that a program denotes under substitution. We will define the unnormalized density to be invariant under substitution, which is to say that
\begin{align*}
    \den{\mpt{f}(c')[\tau']}_\gamma (\tau) &= \gamma_{f[\tau']}(\tau ; c') = \gamma_f(\tau ; c').
\end{align*}
For the prior under substitution we use the notation 
\begin{align*}
    \den{\mpt{f}(c')[\tau']}_p (\tau) = p_{f[\tau']}(\tau ; c'),
\end{align*}
to denote a density over newly sampled variables, rather than over all unobserved variables, 
\[
\begin{prooftree}
  \hypo{
    c,\tau,\rho,w 
    \ceval{\tau'}{\mpt{f($c^\prime$)}}
  }
  \infer1{\displaystyle
          \den{\mpt{f($c^\prime$)}[\tau']}_p(\tau)
          =
          \prod_{\alpha \in \text{dom}(\tau) \setminus \text{dom}(\tau')} 
          \rho(\alpha)}
\end{prooftree}
\]
This construction ensures that $w = \gamma_{f[\tau']}(\tau ; c') / p_{f[\tau']}(\tau ; c')$, as in the case where no substitution is performed. 

As previously, the support $\Omega_{f[\tau']}$ is defined implicitly as the set of traces that can be generated via evaluation under substitution, which is a subset of the original support
\begin{align*}
    \Omega_{f[\tau']}
    =
    \left\{\tau \in \Omega_{f}: \tau(\alpha)\!=\!\tau'(\alpha) \:\forall\: \alpha \in \text{dom}(\tau) \cap \text{dom}(\tau') \right\}.
\end{align*}



\section{Inference Combinators}

\begin{figure}[!t]
    \centering
    \includegraphics[width=\linewidth]{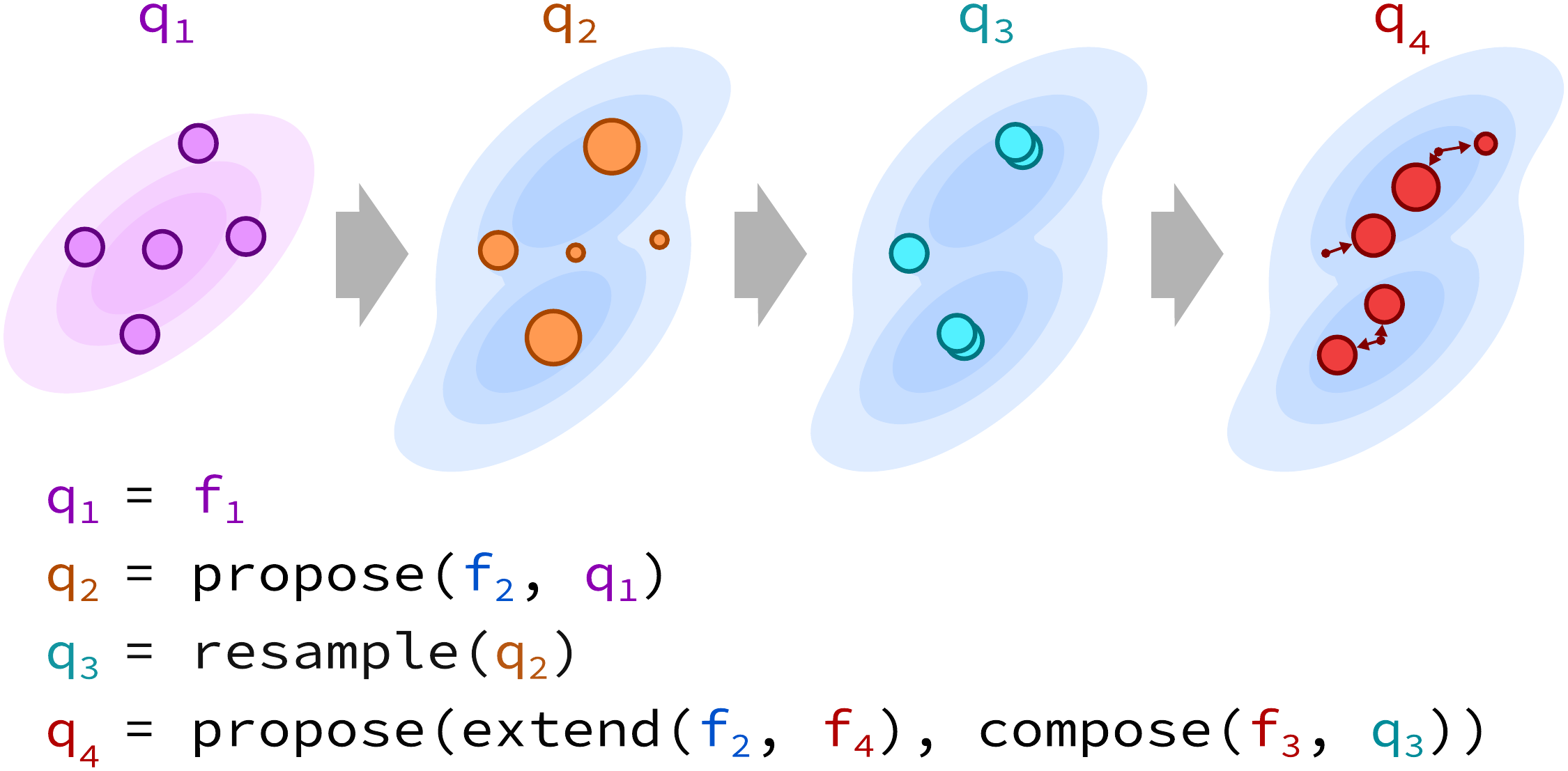}
    \caption{Inference combinators denote fundamental operations in importance sampling, which can be composed to define a sampling strategy. See main text for details. }
    \label{fig:is_ops}
\end{figure}

\vspace{-0.25\baselineskip}
\label{sec:infcomb}
We develop a domain-specific language (DSL) for importance samplers in terms of four combinators, with a grammar that defines their possible compositions
\begin{align*}
    \mpt{f} ::= &~ \text{A primitive program} \\ 
    \mpt{p} ::= &~ \mpt{f} \mid \mpt{extend(p, f)} \\
    \mpt{q} ::= &~ \mpt{p} \mid \mpt{resample(q)} \mid \mpt{compose(q$^\prime$, q)} \mid  \mpt{propose(p, q)}
\end{align*}
We distinguish between three expression types. We use \pt{f} to refer to a \emph{primitive program} in the modeling language, with sampler semantics that perform likelihood weighting as described in the preceding section. We use \pt{p} to refer to a \emph{target program}, which is either a primitive program, or a composition of primitive programs that defines a density on an extended space. Finally, we use \pt{q} to refer to an \emph{inference program} that composes inference combinators.


Figure~\ref{fig:is_ops} shows a simple program that illustrates the use of each combinator, for which we will formally specify semantics below. The first step defines a program \pt{q$_1$} that generates samples from a primitive program \pt{f$_1$}, which are equally weighted (i.e.~\pt{f$_1$} does not have observed variables). The second step defines a program $\mpt{q}_2 = \mpt{propose(f$_2$, q$_1$)}$ that uses samples from \pt{q$_1$} as proposals for the primitive program \pt{f$_2$}, which results in weights that are proportional to the ratio of densities. The third step $\mpt{q}_3 = \mpt{resample(q$_2$)}$ performs importance resampling, which replicates samples with probability proportional to their weights. In the fourth step, we use $\mpt{compose(f$_3$, q$_3$)}$ to apply a program \pt{f$_3$} that denotes a transition kernel, and use the resulting samples as proposals for a program $\mpt{extend(f$_2$, f$_4$)}$, which defines a density on an extended space. To illustrate this extended space construction, we will consider a more representative example of an inference program.


\vspace{-0.5\baselineskip}
\subsection{Example: Amortized Gibbs Sampling}
\vspace{-0.5\baselineskip}

\begin{figure*}[!t]
\begin{minipage}{0.55\linewidth}
\begin{probtorch}
def pop_gibbs(target, proposal, kernels, sweeps):
  q = propose(partial(target, suffix=0), 
              partial(proposal, suffix=0))
  for s in range(sweeps):
    for k in kernels:
      q = propose(
          extend(partial(target, suffix=s+1),
                  partial(k, suffix=s)), 
          compose(partial(k, suffix=s+1), 
                  resample(q, dim=0)))
    return q
\end{probtorch}
\end{minipage}
\begin{minipage}{0.45\linewidth}
\begin{probtorch}
data, opt = ... 
target, proposal, kernels = ... 
q = pop_gibbs(target, proposal, kernels)
for _ in range(10000):
  s0 = State(sample_size=[40, 20], 
             objective=inc_kl)
  s, *outputs = q(s0, data.next_batch(20))
  s.loss.backward()
  opt.step()
  opt.zero_grad()
\end{probtorch}
\end{minipage}
\vspace{-0.5\baselineskip}
\caption{\label{fig:apg-code}
A combinator-based implementation of amortized population Gibbs sampling \citep{wu2020amortized} in python, along with a procedure for training the target model, initial proposal, and each of the kernels that approximate Gibbs conditionals.
In this example \pt{partial} is the partial application function from Python 3's \pt{functools} library.
}
\end{figure*}

\label{sec:example-apg}
Figure~\ref{fig:apg-code} shows a program that makes use of a combinator DSL that is embedded in Python. This code implements amortized population Gibbs (APG) samplers~\citep{wu2020amortized}, a recently developed method that combines stochastic variational inference with sequential Monte Carlo (SMC) samplers to learn a set of conditional proposals that approximate Gibbs updates. 
This is an example where combinators are able to concisely express an algorithm that would be non-trivial to implement from scratch, even for experts.

We briefly discuss each operation in this algorithm. The function \pt{pop_gibbs} (Figure~\ref{fig:apg-code}, left) accepts programs that denote a \texttt{target} density, an initial \pt{proposal}, and a set of \pt{kernels}. It returns a program \pt{q} that performs APG sampling. This program is evaluated in the right column to generate samples, compute an objective, and perform gradient descent to train the initial proposal and the kernels.


APG samplers perform a series of \emph{sweeps}, where each sweep iterates over proposals that approximate Gibbs conditionals. To understand the weight computation in this sampler, we consider Program~\ref{prog:deepmix} as a (non-representative) example. Here updates could take the form of block proposals $q(v \mid x, z)$ and $q(z \mid x, v)$. At each step, we construct an outgoing sample by either updating $v$ or $z$ given an incoming sample $(w,v,z)$. Here we consider an update of the variable $v$,
\begin{align*}
    w' &= \frac{p(x,v',z) \: q(v \mid x, z)}
              {q(v' \mid x,z) \: p(x,v,z)} \: w,
    &
    v' &\sim q(\cdot \mid x, z).
\end{align*}
This weight update makes use of an auxiliary variable construction, in which we compute the ratio between a target density and proposal on an extended space
\begin{align*}
    \tilde{p}(x,v',z,v) &= p(x,v',z) \: q(v \mid x,z), \\
    \tilde{q}(v',z,v \mid x) &= q(v' \mid x,z) \: p(x,v,z).
\end{align*}
In the inner loop for each sweep, we use the \pt{extend} combinator to define the extended target $\tilde{p}(x,v',z,v)$. To generate the proposal, we use the \pt{compose} combinator to combine the incoming sampler with the kernel, which defines the extended proposal $\tilde{q}(v',z,v \mid x)$. Since the marginal $\tilde{p}(x,v',z) = p(x,v',z)$, the outgoing sample $(w',v',z)$ is properly weighted for the target density as long as the incoming sample is properly weighted.

\vspace{-0.5\baselineskip}
\subsection{Programs as Proposals}
\vspace{-0.5\baselineskip}

When we use a program as a proposal for another program, it may not be the case that the proposal and target instantiate the same set of variables. Here two edge cases can arise: (1) the proposal contains \emph{superfluous} variables that are not referenced in the target, or (2) there are variables in the target that are \emph{missing} from the proposal.

To account for both cases, we define an \emph{implicit} auxiliary variable construction. To illustrate this construction, we consider a proposal that defines a density $q(u,z \mid x)$
\begin{probtorch}%[caption=A inference program for Program~\ref{prog:deepmix}, label=prog:mixproposal]

$\lambda^u$, $\lambda^z$ = ... # (initialize encoder networks)
def g(s, x):
  u = s.sample(Normal($\lambda^u_\mu$(x), $\lambda^u_\sigma$(x)),"u")
  z = s.sample(Multinomial(1, $\lambda^z$(u), "z")
  return s, x, u, z
\end{probtorch}
Suppose that we use \pt{propose(f, g)} to associate this proposal with the target in  Program~\ref{prog:deepmix}, which defines a density $p(x,v,z)$. We then have a superfluous variable with address \pt{"u"}, as well as a missing variable with address \pt{"v"}. To deal with this problem, we implicitly extend the target using the conditional density $q(u \mid x)$ from the inference program, and conversely extend the proposal using the conditional density $p(v \mid z)$ from the target program, 
\begin{align*}
    \tilde{p}(x,v,z,u) &= p(x,v,z) \: q(u \mid x), \\
    \tilde{q}(v,z,u \mid x) &= q(u,z \mid x) \: p(v \mid z).
\end{align*}
Since, by construction, the conditional densities for $u$ and $v$ are identical in the target and proposal, these terms cancel when computing the importance weight
\begin{align*}
    w = \frac{\tilde{p}(x,v,z,u)}
             {\tilde{q}(v,z,u \mid x)}
     = \frac{p(x,v,z) \: q(u \mid x)}
            {q(u,z \mid x) \: p(v \mid z)}
     = \frac{p(x \mid v) \: p(z)}
            {q(z \mid u)}
    .
\end{align*}
This motivates a general importance sampling computation for primitive programs of the following form
\begin{align*}
    c_1, \tau_1, \rho_1, w_1 &\eval \mpt{g($c_0$)}, & 
    c_2, \tau_2, \rho_2, w_2 &\ceval{\tau_1}{\mpt{f($c_0$)}}.
\end{align*}
We generate $\tau_1$ from the proposal, and then use substitution to generate a trace $\tau_2$ that reuses as many variables from $\tau_1$ as possible, and samples any remaining variables from the prior. Here the set of missing variables is $\text{dom}(\tau_2) \setminus \text{dom}(\tau_1)$ and the set of superfluous variables is $\text{dom}(\tau_1) \setminus \text{dom}(\tau_2)$. 
Hence, we can define the importance weight
\begin{align}
    \label{eq:propose-weight}
    \begin{split}
    w &= \frac{\gamma_f(\tau_2 ; c_0) \: p_{g[\tau_2]}(\tau_1 ; c_0)}
              {\gamma_g(\tau_1 ; c_0) \: p_{f[\tau_1]}(\tau_2 ; c_0)} 
         \: w_1,\\
      &= \frac{w_2}
              {\prod_{\alpha \in \text{dom}(\rho_1) 
                      \setminus (\text{dom}(\tau_1) 
                                 \setminus \text{dom}(\tau_2))} 
              \rho_1(\alpha)}
         \: w_1.
    \end{split}
 \end{align}
 In the numerator, we take the product over all terms in the target density, exclusive of missing variables. Note that this expression is identical to $w_2$ (Equation~\ref{eq:conditioned-weight}).
 In the denominator, we take the product over all terms in the proposal density $\rho_{1}$,
exclusive of superfluous variables $\text{dom}(\tau_1) \setminus \text{dom}(\tau_2)$.

\vspace{-0.5\baselineskip}
\subsection{Properly Weighted Programs}
\vspace{-0.5\baselineskip}

The expression in Equation~\ref{eq:propose-weight} is not just valid for a composition \pt{propose(f, g)} in which  \pt{f} and \pt{g} are primitive programs. As we will show, this expression also applies to any composition \pt{propose(p, q)}, in which \pt{p} is a target program, and \pt{q} is itself an inference program. While this distinction may seem subtle, it is fundamental; it allows us to use any sampler \pt{q} as a proposal, which yields a new sampler that can once again be used as a proposal.

To demonstrate the validity of combinator composition, we make use of the framework of nested importance samplers and proper weighting \citep{liu2008monte, naesseth2015nested}, which we extend to evaluation of probabilistic programs.

\begin{definition}[Properly Weighted Evaluation]
    \label{def:proper_weighting}
    \itshape
    Let \pt{q($c^\prime$)} denote an unnormalized density $\den{\mpt{q($c^\prime$)}} = \gamma_{q}(\cdot \,; c')$. Let $\pi_{q}(\cdot ; c')$ denote the corresponding probability density and let $Z_{q}(c')$ denote the normalizing constant such that 
    \begin{align*}
        \pi_{q}(\tau ; c') &:= \frac{\gamma_{q}(\tau ; c')}{Z_{q}(c')},
        &
        Z_{q}(c') &:= \int \! d\tau \: \gamma_{q}(\tau ; c').
    \end{align*}
    We refer to the evaluation of $c,\tau,\rho,w \eval  \mpt{q($c^\prime$)}$ as properly weighted for its unnormalized density $\gamma_{q}(\cdot \,; c')$ when, for all measurable functions $h$
    \begin{align*}
        \mathbb{E}_{\mpt{q($c^\prime$)}}[w \: h(\tau)] 
        &= \mathcal{C}_q(c') \int \! d\tau \: \gamma_{q}(\tau ; c') \: h(\tau), \\
        &= \mathcal{C}_q(c') \: Z_{q}(c') \: \mathbb{E}_{\pi_q(\cdot ; c')}[h(\tau)]
        ,
    \end{align*}
    for some constant $\mathcal{C}_q(c')>0$.
    When $\mathcal{C}_q(c')=1$, we refer to the evaluation as strictly properly weighted.    
\end{definition} 
A properly weighted evaluation can be used to define self-normalized estimators of the form in Equation~\ref{eq:self-normalized}. Such estimators are strongly consistent, which is to say that they converge almost surely in the limit of infinite samples $L$
\begin{align*}
    \frac{
        \frac{1}{L}
        \sum_{l=1}^L w^l \: h(\tau^l)
    }{
        \frac{1}{L}
        \sum_{l=1}^L w^l
    }
    \,\stackrel{a.s.}{\longrightarrow}\,
    \mathbb{E}_{\pi_{q}(\cdot ; c')}\left[h(\tau) \right]
    .
\end{align*}
\begin{proposition}[Proper Weighting of Primitive Programs]\label{prop:primitive_proper_weighting} Evaluation of a primitive program \pt{f($c'$)} is strictly properly weighted for its unnormalized density $\den{\mpt{f($c^\prime$)}}_\gamma$. 
\end{proposition}

\begin{proof}
This holds by definition, since in a primitive program $w$ is uniquely determined by $\tau$, which is a sample from the prior density $\den{\mpt{f($c^\prime$)}}_p = p_f(\cdot \: ; c')$,
\begin{align*}
    \mathbb{E}_{\mpt{f($c^\prime$)}}
    \left[
        w\
        h(\tau)
    \right]
    &=
    \mathbb{E}_{p_f(\cdot; c')}
    \left[
        \frac{
            \gamma_{f}(\tau; c')
        }{
            p_f(\tau; c')
        }
        \:
        h(\tau)
    \right],
    \\
    &=     
    Z_f(c') \:
    \mathbb{E}_{p_f(\cdot; c')}
    \left[
        \frac{
            \pi_{f}(\tau; c')
        }{
            p_f(\tau; c')
        }
        \:
        h(\tau)
    \right],
    \\
    &=
    Z_f(c') \:
    \mathbb{E}_{\pi_f(\cdot; c')}
    \left[
        h(\tau)
    \right]
    .
\end{align*}
\par\vspace{-2\baselineskip}
\end{proof}


\vspace{-0.75\baselineskip}
\subsection{Operational Semantics}
\vspace{-0.5\baselineskip}
\label{subsec:operational}
Given a modeling language for primitive programs whose evaluation is strictly properly weighted, our goal is to show that the inference language preserves strict proper weighting. To do so, we begin by formalizing rules for evaluation of each combinator, which together define big-step operational semantics for the inference language.

\paragraph{Compose.} We begin with the rule for \pt{compose(q$_2$,q$_1$)}, which performs program composition.
\[
\begin{prooftree}
  \hypo{
    \begin{matrix}
      c_1, \tau_1, \rho_1, w_1 \eval \mpt{q$_1$($c_0$)} \quad
      c_2, \tau_2, \rho_2, w_2 \eval \mpt{q$_2$($c_1$)} \\[3pt]
      \text{dom}(\rho_1) \cap \text{dom}(\rho_2) = \emptyset
    \end{matrix}
  }
  \infer1{c_2, \tau_2 \oplus \tau_1, \rho_2 \oplus \rho_1,~
          w_2 \cdot w_1 
          ~\eval{}~ \mpt{compose(q$_2$, q$_1$)($c_0$)}}
\end{prooftree}
\]
This rule states that we can evaluate \pt{compose(q$_2$, q$_1$)($c_0$)} by first evaluating \pt{q$_1$($c_0$)} and using the returned values $c_1$ as the inputs when evaluating \pt{q$_2$($c_1$)}. We return the resulting value $c_2$ with weight $w_2 \cdot w_1$. We combine traces $\tau_2 \oplus \tau_1$ and density maps $\rho_2 \oplus \rho_1$ using the operator $\oplus$, which we define for maps $\mu_1$ and $\mu_2$ with disjoint domains as
\begin{align*}
    (\mu_1 \oplus \mu_2)(\alpha) 
    = 
    \begin{cases}
        \mu_1(\alpha) & \alpha \in \text{dom}(\mu_1), \\
        \mu_2(\alpha) & \alpha \in \text{dom}(\mu_2). \\
    \end{cases}
\end{align*}

\paragraph{Extend.} The combinator \pt{extend(p, f)} performs a composition between a target \pt{p} and a primitive \pt{f} which defines a density on an extended space.
\[
\begin{prooftree}
  \hypo{
    \begin{matrix}
      c_1, \tau_1, \rho_1, w_1 \eval \mpt{p($c_0$)} \qquad
      c_2, \tau_2, \rho_2, w_2 \eval \mpt{f($c_1$)} \\[3pt]
      \text{dom}(\rho_1) \cap \text{dom}(\rho_2) = \emptyset 
      \qquad
      \text{dom}(\rho_2) = \text{dom}(\tau_2)
    \end{matrix}
  }
  \infer1{c_2, \tau_1 \oplus \tau_2, \rho_1 \oplus \rho_2, w_1 \cdot w_2
          ~\eval{}~ \mpt{extend(p, f)($c_0$)}}
\end{prooftree}
\]
The program \pt{f} defines a ``kernel'', which may not contain observed variables. We enforce this by requiring that $\text{dom}(\rho_2)=\text{dom}(\tau_2)$.

\paragraph{Propose.} The extend operator serves to incorporate auxiliary variables into a target density. When evaluating \pt{propose(p, q)}, we discard auxiliary variables to continue the inference computation. For this purpose, we define a transformation \pt{marginal(p)} to recover the original un-extended program. We define this transformation recursively 
\[
\begin{prooftree}
    \hypo{
      \phantom{\mpt{f}' = \mpt{marginal(p)}}
    }
    \infer1{
        \mpt{f} = \mpt{marginal(f)} 
    }
\end{prooftree}
\qquad
\begin{prooftree}
    \hypo{
      \mpt{f}' = \mpt{marginal(p)}
    }
    \infer1{
        \mpt{f}' = \mpt{marginal(extend(p, f))} 
    }
\end{prooftree}
\]
We now define the operational semantics for propose as
\[
\begin{prooftree}
  \hypo{
    \begin{matrix}
      c_1, \tau_1, \rho_1, w_1 \eval \mpt{q($c_0$)} \quad
      c_2, \tau_2, \rho_2, w_2 \: \ceval{\tau_1}{\mpt{p($c_0$)}} \\[6pt]
      c_3, \tau_3, \rho_3, w_3 \: \ceval{\tau_2}{\mpt{marginal(p)($c_0$)}} \\[6pt]
      \displaystyle
      u_1 = \prod_{\alpha \in 
                          \text{dom}(\rho_1) 
                             \setminus (\text{dom}(\tau_1) 
                                        \setminus \text{dom}(\tau_2))}
                 \rho_1(\alpha)
    \end{matrix}
  }
  \infer1{c_3, \tau_3, \rho_3,~
          w_2 \cdot w_1 / u_1
          ~\eval{}~ \mpt{propose(p, q)($c_0$)}}
\end{prooftree}
\]
In this rule, the outgoing weight $w_2 \cdot w_1 / u_1$ corresponds precisely to Equation~\ref{eq:propose-weight}, since $w_2$ is equal to the numerator, whereas $u_1$ is equal to the denominator in this expression. 

Evaluation of \pt{p($c_0$)}$[\tau_1]$ applies substitution recursively to sub-expressions (see Appendix~\ref{apx:substitution}). Note that, by construction, evaluation of \pt{marginal(p)($c_0$)}$[\tau_2]$ is deterministic, and that $\tau_3$ and $\rho_3$ correspond to the entries in $\tau_2$ and $\rho_2$ that are associated with the unextended target.




\paragraph{Resample.} This combinator performs importance resampling on the return values, the trace, and the density map. Since resampling is an operation that applies to a collection of samples, we use a notational convention in which $c^l$, $\tau^l$, $\rho^l$, $w^l$ refer to elements in vectorized objects, and $\v{c}$, $\v{\tau}$, $\v{\rho}$, $\v{w}$ refer to vectorized objects in their entirety\footnote{For simplicity we describe resampling with a single indexing dimension. The Probabilistic Torch implementation supports tensorized evaluation. For this reason, we specify a dimension along which resampling is to be performed in Figure~\ref{fig:apg-code}.}. Resampling as applied to a vectorized program has the semantics 
\[
\begin{prooftree}
  \hypo{
    \begin{matrix}
      \v{c}_1, \v{\tau}_1, \v{\rho}_1, \v{w}_1 \eval \mpt{q($\v{c}_0$)} \quad
      \v{a}_1 \sim \textsc{resample}(\v{w}_1) \\[3pt]
      \v{c}_2, \v{\tau}_2, \v{\rho}_2 = \textsc{reindex}(\v{a}_1, \v{c}_1, \v{\tau}_1, \v{\rho}_1) \qquad
      \v{w}_2 = \textsc{mean}(\v{w}_1)
    \end{matrix}
  }
  \infer1{\v{c}_2, \v{\tau}_2, \v{\rho}_2, \v{w}_2
          ~\eval{}~ \mpt{resample(q)($\v{c}_0$)}}
\end{prooftree}
\]
In this rule, we make use of three operations. The first samples indices with probability proportional to their weight using a random procedure $\v{a}_1 \sim \textsc{resample}(\v{w}_1)$\footnote{We use systematic resampling in our implementation. For a comparison of methods see \citep{murray2016parallel}}.
We then use a function $\v{c}_2, \v{\tau}_2, \v{\rho}_2 =\textsc{reindex}(\v{a}_1, \v{c}_1, \v{\tau}_1, \v{\rho}_1)$ to replicate objects according to the selected indices, 
\begin{align}
    c^l_2 &= c_1^{a^l_1},
    &
    \tau^l_2 &= \tau_1^{a^l_1},
    &
    \rho^l_2 &= \rho_1^{a^l_1}.
\end{align}
Finally, we use $\v{w}_2=\textsc{mean}(\v{w}_1)$ to set outgoing weights to the average of the incoming weights, $w^l_2 = \sum_{l'} w^{l'}_1 / L $.

\paragraph{Denotational Semantics.} To demonstrate that program \pt{q} are (strictly) properly weighted, we need to define the density that programs \pt{p} and \pt{q} denote. Owing to space limitations, we relegate discussion of these denotational semantics to Appendix~\ref{apx:denotational_semantics}. Definitions follow in a straightforward manner from the denotational semantics of primitive programs.

\paragraph{Proper Weighting.} With this formalism in place, we now state our main claim of correctness for samplers that are defined in the inference language.

\begin{theorem}[Strict Proper Weighting of Inference Programs] Evaluation of an inference program \pt{q($c$)} is strictly properly weighted for its unnormalized density $\den{\mpt{q($c$)}}_\gamma$.
\end{theorem}
\vspace{-0.5\baselineskip}
We provide a proof in Appendix~\ref{apx:proper_weighting_proofs}, which is by induction from lemmas for each combinator. 

\vspace{-0.5\baselineskip}
\section{Learning Neural Proposals}
\label{sec:neural_proposals}
\vspace{-0.5\baselineskip}


To learn a proposal \pt{q}, we use properly-weighted samples to compute variational objectives. We consider three objectives for this purpose. The first two minimize a top-level reverse or forward KL divergence, which corresponds to performing stochastic variational inference \citep{wingate2013automated} or reweighted wake-sleep style inference \citep{le2019revisiting}. The third implements nested variational inference \citep{zimmermann2021nested} by defining an objective at each level of recursion. We describe the loss and gradient computations at a high level, and provide details in
Appendix~\ref{apx:objective_computation} and Appendix~\ref{apx:gradient_computation} respectively.

\paragraph{Objective computation.}
To optimize the parameters of proposal programs, we slightly modify the operational semantics (for details see Appendix~\ref{apx:objective_computation}) such that a user-defined objective function $\ell : (\rho_q, \rho_p, w, v) \to \mathbb{R}$ can be evaluated in the context of the \pt{propose} combinator.
Objective functions are defined in terms of the density maps of the proposal and target program $\rho_q$, $\rho_p$ and the incoming and incremental importance weights $w$, $v$ at the current level of nesting.
The \emph{local} losses computed at the individual levels of nesting are accumulated to a \emph{global} loss in the inference state which consecutively can be used to compute gradients w.r.t.~parameters of the target and proposal and programs.

\paragraph{Stochastic Variational Inference (SVI).} 

Suppose that we have a program \pt{q$_2$ $=$ propose(p, q$_1$)} in which the target and proposal have parameters $\theta$ and $\phi$,
\begin{align}
    \den{\mpt{p(c$_0$)}}_\gamma &= \gamma_p(\cdot \:; c_0, \theta),
    &
    \den{\mpt{q$_1$(c$_0$)}}_\gamma &= \gamma_{q}(\cdot \:; c_0, \phi).
\end{align}
The target program \mpt{p} and the inference program \mpt{q$_2$} denote the same density $\den{\mpt{p(c$_0$)}}_\gamma = \den{\mpt{q$_2$(c$_0$)}}_\gamma$ 
and hence, as a result of Theorem \ref{thm:target_proper_weighting}, the evaluation of \mpt{q$_2$} is strictly properly weighted for $\gamma_p$.
Hence, we can evaluate $c_2,\tau_2,\rho_2,w_2 \eval \mpt{q$_2$(c$_0$)}$ to compute a stochastic lower bound \citep{burda2016importance},
\begin{align}
    \label{eq:svi-bound}
    \mathcal{L(\theta, \phi)}
    &:=
    \mathbb{E}_{\mpt{q$_2$($c_0$)}}
    \left[
    \log w_2
    \right]
    \\
    &\le
    \log 
    \left(
        \mathbb{E}_{\mpt{q$_2$($c_0$)}}
        \left[
            w_2
        \right]
    \right)
    =
    \log Z_p(c_0, \theta)
    \nonumber
    ,
\end{align}
where the penultimate equality holds by Definition~\ref{def:proper_weighting}.
The gradient $\nabla_\theta \mathcal{L}$ of this bound is a biased estimate of $\nabla_\theta \log Z_p$. The gradient $\nabla_\phi \mathcal{L}$ can be approximated using likelihood-ratio estimators \citep{wingate2013automated,ranganath2014black}, reparameterized samples \citep{kingma2013auto-encoding,rezende2014stochastic}, or a combination of the two \citep{ritchie2016deep}. In the special case where $\mpt{q}_1 = \mpt{f}$ is a primitive program, the  gradient of the reverse KL-divergence between \pt{p} and \pt{q$_1$} is
\begin{align*}
    \nabla_\phi \mathcal{L}(\theta, \phi) 
    &=
    \mathbb{E}_{\mpt{q$_2$($c_0$)}}
       \left[
           \frac{\partial \log w_2}{\partial \tilde\tau_2}
           \frac{\partial \tilde\tau_2}{\partial \phi}
           +
           \frac{\partial}{\partial \phi}
           \log w_2
       \right],
    \\ 
    &= 
    - \nabla_\phi \text{KL}(\pi_f || \pi_p).
\end{align*}


\paragraph{Reweighted Wake-sleep (RWS) Style Inference.} To implement variational methods inspired by reweighted wake-sleep, we use samples $c_2^l,\tau_2^l,\rho_2^l,w_2^l \eval \mpt{q$_2$($c_0$)}$ to compute a self-normalized estimate of the gradient
\begin{align}
    \begin{split}
    \label{eq:rws-log-Z}
    \nabla_\theta \log Z_p(c_0, \theta) 
    &= 
    \mathbb{E}_{\pi_p(\tau \:; c_0, \theta)} \left[ 
      \nabla_\theta \log \gamma_p(\tau ; c_0, \theta)
    \right],\\
    &\simeq \sum_{l} \frac{w_2^l}{\sum_{l'} w_2^{l'}}
    \nabla_\theta \log \gamma_p(\tau_2^l ; c_0, \theta).
    \end{split}
\end{align}
Notice that here we compute the gradient w.r.t.~ the non-extended density $\gamma_p$, which does not include auxiliary variables and hence density terms which would integrate to one. When approximating gradient this allows us to compute lower variance estimates.

Similarly, we approximate the gradient of the forward KL divergence with a self-normalized estimator that is defined in terms of the proposals $c_1^l, \tau_1^l,\rho_1^l,w_1^l \eval \mpt{q$_1$($c_0$)}$,
\begin{align}  
    \label{eq:rws-kl}
    &-\!\nabla_\phi 
    \text{KL}(\pi_p || \pi_{q})
    = 
    \mathbb{E}_{\pi_p(\tau \:; c_0, \theta)} \!\!\left[ 
      \nabla_\phi \log \pi_{q}(\tau ; c_0, \phi)
    \right],\\
    \nonumber
    &\qquad\simeq 
     \sum_{l} \left(
      \frac{w_2^l}{\sum_{l'} w_2^{l'}}
      - \frac{w_1^l}{\sum_{l'} w_1^{l'}}
     \right)
    \nabla_\phi \log \gamma_{q}(\tau_1^l ; c_0, \phi).
\end{align}
In the special case where the proposal $\mpt{q}_1 = \mpt{f}$ is a primitive program without observations (i.e. $w^l_1=1$), we can drop the second term to recovers the standard RWS estimator. 


\begin{figure}[!t]
    \centering
    \includegraphics[width=\linewidth]{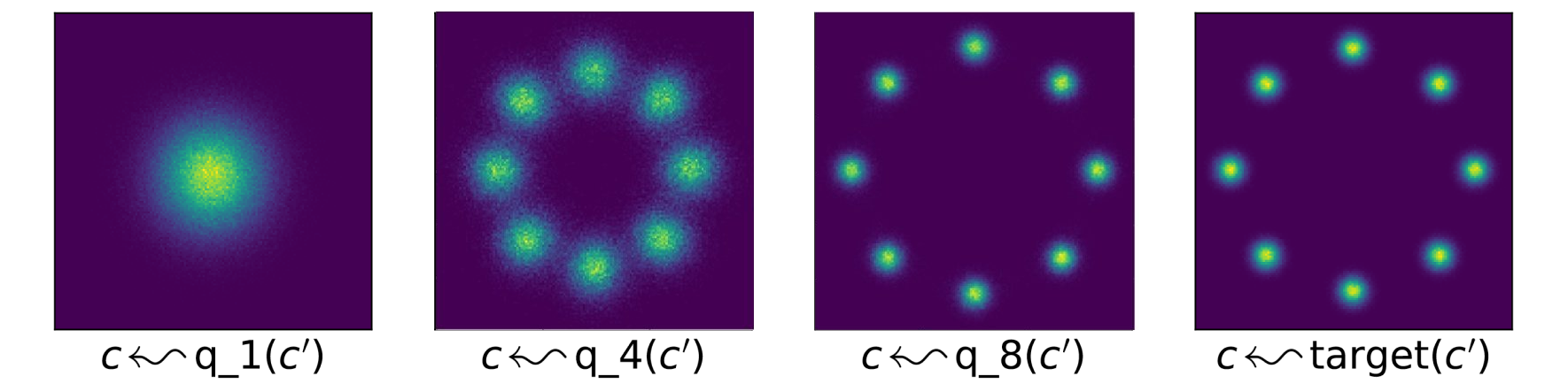}
    \vspace{-1.5\baselineskip}
    \caption{
        Samples from the initial proposal \pt{q}$_1$, learned intermediate proposal \pt{q}$_4$, final proposal \pt{q}$_8$, and final \texttt{target}.
        Definitions can be found in Figure \ref{fig:nvi-code}, Appendix \ref{apx:avi}.
    }
    \vspace{-0.5\baselineskip}
    \label{fig:final-nvi-samples}
\end{figure}

\begin{figure}[!t]
    \centering
    \includegraphics[width=\linewidth]{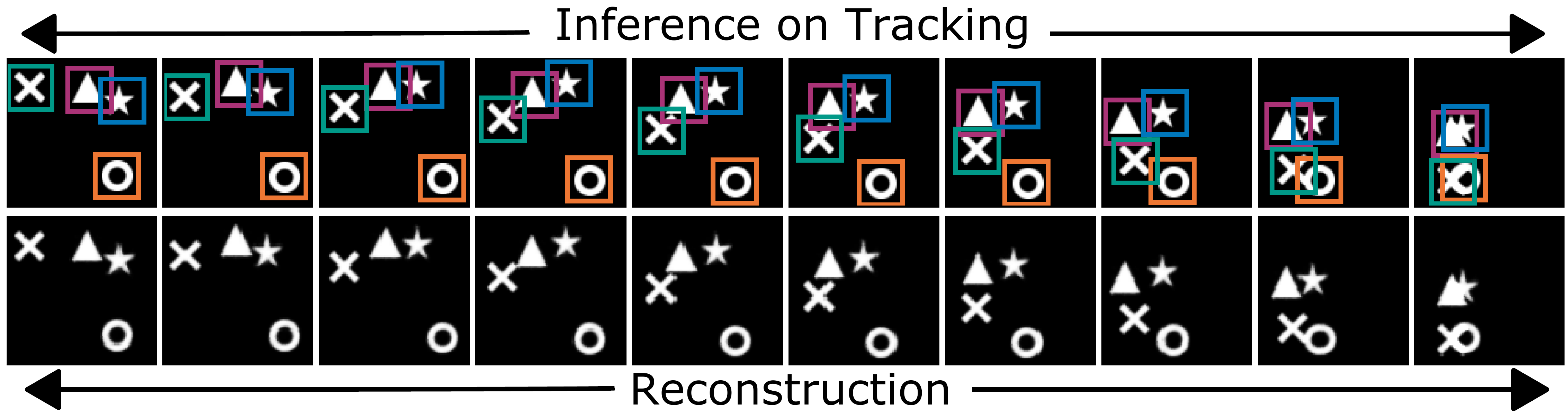}
    \caption{Qualitative results for the tracking task: Top row shows inferred positions of objects, bottom row shows reconstruction of the video frames.}
    \label{fig:apg-samples}
\end{figure}

\paragraph{Nested Variational Inference.}
A limitation of both SVI and RWS is that they are not well-suited to learning parameters in samplers at multiple levels of nesting. To see this, let us consider a program
\begin{probtorch}
    q$_3$ = propose(p$_3$, propose(p$_2$, f$_1$))
\end{probtorch}
We denote $\den{\mpt{q}_3}_\gamma=\den{\mpt{p}_3}_\gamma=\gamma_3$, $\den{\mpt{p}_2}=\gamma_2$, and $\den{\mpt{f}_1}=\gamma_1$. For simplicity, we consider densities with constant support $\Omega_p = \Omega_2 = \Omega_1$. The top-level evaluation $c_3, \tau_3, \rho_3, w_3 \eval \mpt{q$_3$($c_0$)}$ then yields a trace $\tau_3=\tau_1$ that contains samples from the prior $\den{\mpt{f}_1(c_0)} = p_1(\cdot \:; c_0)$ with weight
\begin{align}
    w_3 
    = 
    \frac{\gamma_3(\tau_1 ; c_0, \theta)}
         {\cancel{\gamma_2(\tau_1 ; c_0, \phi_2)}}
    \frac{\cancel{\gamma_2(\tau_1 ; c_0, \phi_2)}}
         {\cancel{\gamma_1(\tau_1 ; c_0, \phi_1)}}
    \frac{\cancel{\gamma_1(\tau_1 ; c_0, \phi_1)}}
         {p_1(\tau_1 ; c_0, \phi_1)}
    .
\end{align}
For this program, the lower bound from Equation~\ref{eq:svi-bound} does not depend on $\phi_2$. Conversely, the RWS-style estimator from Equation~\ref{eq:rws-kl} does not depend on $\phi_1$. Analogous problems arise in non-trivial inference programs that incorporate transition kernels and resampling. One such example are SMC samplers, where we would like to learn a sequence of intermediate densities, which impacts the variance of the final sampler. We consider this scenario Section  \ref{sec:experiments}.

Nested Variational Inference (NVI) \citep{zimmermann2021nested} replaces the top-level objectives in SVI and RWS with an objective that contains one term at each level of nesting.
In the example above, this objective has the form
\begin{align*}
    \mathcal{D}
    &=
    D_3(\pi_2 \,||\, \pi_3)
    +
    D_2(\pi_1 \,||\, \pi_2)
    +
    D_1(p_1 \,||\, \pi_1)
    ,
\end{align*}
where each $D_i$ is a forward or a reverse KL divergence or a corresponding stochastic upper or lower bound.
The individual terms can be optimized as described above, but additional care has to be taken when optimizing the intermediate target densities, as their normalizing constants might not be tractable.
For details on the computation of the nested variational objective and corresponding gradients we refer to Appendix~ \ref{apx:gradient_computation}.
When we apply an NVI objective using an upper bound based on the forward KL divergence to the program in Figure~\ref{fig:apg-code}, we recover the gradient estimators of APG samplers.


\vspace{-1.25\baselineskip}
\section{Experiments}
\label{sec:experiments}
\vspace{-0.5\baselineskip}

We evaluate combinators in two experiments. First, we learn proposals in an annealed importance sampler that additionally learns intermediate densities. Second, we use an APG sampler to learn proposals and a generative model for an unsupervised multi-object tracking task. 

\vspace{-0.5em}
\paragraph{Annealed Variational Inference.} 
We consider the task of generating samples from a 2-dimensional unnormalized density consisting of 8 modes, equidistantly placed along a circle. 
For purposes of evaluation we treat this density as a blackbox, which we are only able to evaluate pointwise.

We implement an annealed importance sampler \citep{neal2001annealed} (Figure~\ref{fig:nvi-code} in Appendix~\ref{apx:avi}) for a sequence of unnormalized densities $\gamma_k(\tau_k ; c_0) = \gamma_K(\tau_k ; c_0)^{\beta_k} \gamma_1(\tau_k ; c_0)^{1-\beta_k}$ that interpolate between an initial proposal $\gamma_{1}$ and the final target $\gamma_K$.
The sampler employs forward kernels $q_k(\tau_k ; c_{k-1}, \phi_k)$ and reverse kernels $r_k(\tau_{k-1} ; c_k, \theta_k)$ to define densities on an extended space at each level of nesting. 
We train the sampler using NVI by optimizing the kernel parameters $\theta_k$ and $\phi_k$, and parameters for the annealing schedule $\beta_k$. 
%


We compare NVI and NVI*, which additionally learns the annealing schedule for the intermediate targets,
and corresponding versions, NVIR and NVIR*, which additional employ resampling at every level of nesting, to annealed variational objectives (AVO) \citep{huang2018improving}. 
Learning the annealing schedule (NVI(R)*) results in improved sample quality, in terms of the log expected weight ($log \hat Z$) and the effective sample size (ESS). We report results in Table~\ref{tab:nvi} and refer to Appendix~\ref{apx:avi} for additional details.

\setlength{\tabcolsep}{3.5pt}
\begin{table}[!t]
  \footnotesize
  \centering
  \caption{
      AVO and NVI variants trained for different numbers of annealing steps $K$ and samples per step $L$ for a fixed sampling budget of $K \cdot L = 288$ samples.
  }
  \vspace*{-0.15\baselineskip}
  \begin{tabular}{lccccccccc}   
      \toprule
      &\multicolumn{4}{c}{$\textbf{log} \: \hat{\bm{Z}} = \log \left(\frac{1}{L} \sum_{l} w^{l}\right)$}&
      &\multicolumn{4}{c}{\textbf{ESS}}\\
      \cmidrule(lr){2-5}
      \cmidrule(lr){6-10}
      & K=2 & K=4 & K=6 & K=8 &
      & K=2 & K=4 & K=6 & K=8
      \\
      \hline 
      AVO 
      & 1.88 & 1.99 & 2.05 & 2.06 &
      & 426 & 291 & 285 & 295
      \\
      \hline
      NVI
      & 1.88 & 1.99 & 2.06 & 2.07 &
      & 427 & 341 & 308 & 319
      \\
      NVIR
      & 1.88 & \textbf{2.05} & 2.07 & \textbf{2.08} &
      & 418 & 828 & 934 & 961
      \\
      NVI*
      & 1.88 & 2.03 & \textbf{2.08} & \textbf{2.08} &
      & 427 & 304 & 414 & 481
      \\
      NVIR* 
      & 1.88 & 1.99 & \textbf{2.08} & \textbf{2.08} &
      & 418 & \textbf{981} & \textbf{978} & \textbf{965} 
      \\
      \bottomrule
  \end{tabular}
  \label{tab:nvi}
\end{table}
\setlength{\tabcolsep}{3.5pt}
\begin{table}[!t]
    \footnotesize
    \centering
    \caption{$\log p_\theta (x, z)$ on test sets that contain D objects and T time steps. $L$ is the number of particles and $K$ is the number of sweeps. We run APG and RWS with computational budget $L\cdot K=200$, and run HMC-RWS with $L\cdot K=4000$.}
    \begin{tabular}{llcccc}
    \toprule
    & & \multicolumn{2}{c}{D=3} &  \multicolumn{2}{c}{D=4}\\
    \cmidrule(lr){3-4}
    \cmidrule(lr){5-6}              
    \textbf{Model} & \textbf{Budget} &  T=10 & T=20 &  T=10 & T=20 \\
    \midrule 
    RWS & L=200, K=1 & -5247 & -9396 & -8275 & -16070 \\
    HMC-RWS & L=200, K=20 & -5137 & -9281 &  -8124 & -15087 \\
    APG & L=100, K=2 & -2849 & -5008 &-4411 & -8966 \\
    APG & L=40, \hspace{0.25em} K=5 & -2300 & -4646 & -3529 & -6879 \\
    APG & L=20, \hspace{0.25em} K=10 & \textbf{-2267} & \textbf{-4606} & \textbf{-3516} & \textbf{-6827} \\
    \bottomrule
    \end{tabular}
\label{tab:apg}
\vspace*{0.1\baselineskip}
\end{table}

\vspace{-0.5em}
\paragraph{Amortized Population Gibbs.} In this task, the data is a corpus of simulated videos that each contain multiple moving objects. Our goal is to learn both the target program (i.e.~the generative model) and the inference program using the APG sampler that we introduced in Section~\ref{sec:example-apg}. See Appendix~\ref{apx:apg_implementation} for implementation details.


Figure~\ref{fig:apg-samples} shows that the APG sampler can (fully unsupervised) identify, track, and reconstruct individual object in each frame. In Table~\ref{tab:apg} we compare the APG sampler against an RWS baseline and a hand-coded HMC-RWS method, which improves upon RWS proposals using Hamiltonian Monte Carlo. 
Table~\ref{tab:apg} shows that APG outperforms both baselines. Moreover, for a fixed computational budget, increasing the number of sweeps improves  sample quality. 


\vspace{-0.75\baselineskip}
\section{Related Work}
\vspace{-0.5\baselineskip}

This paper builds directly on several lines of work, which we discuss in detail in Appendix~\ref{apx:related}. 

Traced evaluation (Section 2) has a long history in systems that extend general-purpose programming languages with functionality for probabilistic modeling and inference \citep{wingate2011lightweight,mansinghka2014venture,goodman2014design,tolpin2016design}, including recent work that combines probabilistic programming and deep learning \citep{tran2016edward,ritchie2016deep,siddharth2017learning,bingham2018pyro,baydin2018pyprob}. 

The idea of developing abstractions for inference programming has been around for some time \citep{mansinghka2014venture}, and several instantiations of such abstractions have been proposed in recent years \citep{cusumano-towner2019gen,scibior2017denotational,obermeyer2019functional}. The combinator-based language that we propose here is inspired directly by the work of \citet{naesseth_elements_2019} on nested importance sampling, as well as on a body of work that connects importance sampling and variational inference.



\vspace{-0.75\baselineskip}
\section{Discussion}
\label{sec:discussion}
\vspace{-0.5\baselineskip}

We have developed a combinator-based language for importance samplers that are valid by construction, in the sense that samples are properly weighted for the density that a program denotes. We define semantics for these combinators and provide a reference implementation in Probabilistic Torch. Our experiments demonstrate that user-programmable samplers can be used as a basis for sophisticated variational methods that learn neural proposals and/or deep generative models. Inference combinators, which can be implemented as a DSL that extends a range of existing systems, hereby open up opportunities to develop novel nested variational methods for probabilistic programs.


\section*{Acknowledgements}
This work was supported by the Intel Corporation, the 3M Corporation, NSF award 1835309, startup funds from Northeastern University, the Air Force Research Laboratory (AFRL), and DARPA. We would like to thank Adam \'Scibior for helpful discussions regarding the combinator semantics.




\printbibliography
\end{refsection}

\appendix
\onecolumn

\begin{refsection}
\section{Related Work}
\label{apx:related}
\vspace{-0.25\baselineskip}
\subsection{Importance Sampling and MCMC in Variational Inference}
\vspace{-0.5\baselineskip}

This work fits into a line of recent methods for deep generative modeling that seek to improve inference quality in stochastic variational methods. We briefly review related literature on standard methods before discussing advanced methods that combine variational inference with importance sampling and MCMC, which inspire this work.

Variants of SVI maximize a lower bound, which equates to minimizing a reverse KL divergence, whereas methods that derive from RWS minimize an upper bound, which equates to minimizing a forward KL divergence. In the case of SVI, gradients can be approximated using reparameterization, which is widely used in variational autoencoders \citepApp{kingma2013auto-encoding,rezende2014stochastic} or by using likelihood-ratio estimators \citepApp{wingate2013automated,ranganath2014black}. In the more general case, where the model contains a combination of reparameterized and non-reparameterized variables, the gradient computation can be formalized in terms of stochastic computation graphs \citepApp{schulman2015gradient}. Work by \citeApp{ritchie2016deep} explains how to operationalize this computation in probabilistic programming systems. In the case of RWS, gradients can be computed using simple self-normalized estimators, which do not require reparameterization \citepApp{bornschein2015reweighted}. This idea was recently revisited in the context of probabilistic programming systems by \citepApp{le2019revisiting}, who demonstrate that RWS-based methods are often competitive with SVI.

\paragraph{Importance-weighted Variational Inference.} There is a large body of work that improves upon standard SVI by defining tighter bounds, which results in better gradient estimates for the generative model. Many of these approaches derive from importance-weighted autoencoders (IWAEs) \citepApp{burda2016importance}, which use importance sampling to define a stochastic lower bound $\mathbb{E}[\log \hat{Z}] \le \log Z$ based on an unbiased estimate $\mathbb{E}[\hat{Z}]=Z$ (see Section~\ref{apx:avi}). Since any strictly properly weighted importance sampler can be used to define an unbiased estimator $\hat{Z} = \frac{1}{L} \sum_{l} w^l$, this gives rise to many possible extensions, including methods based on SMC \citepApp{le2018auto-encoding,naesseth2018variational,maddison2017filtering} and thermodynamic integration \citepApp{masrani2019thermodynamic}. \citeApp{salimans2015markov} derive a stochastic lower bound for variational inference which uses an importance weight defined in terms of forward and reverse kernels in MCMC steps.
\citeApp{caterini2018hamiltonian} extend this work by optimally selecting reverse kernels (rather than learning them) using inhomogeneous Hamiltonian dynamics. The work on AVO
\citepApp{huang2018improving} learn a sequence of transition kernels that performs annealing from the initial encoder to the posterior, which we use as a baseline in our annealing experiments.

A somewhat counter-intuitive property of these estimators is tightening the bound typically improves the quality of gradient estimates for the generative model, but can adversely affect the signal-to-noise ratio of gradient estimates for the inference model \citepApp{rainforth2018tighter}. These issues can be circumvented by using RWS-style estimators\footnote{Note that the gradient estimate for the generative model in Equation~\ref{eq:rws-log-Z} is identical to the gradient estimate of the corresponding IWAE bound, so these two approaches only differ in the gradient estimate that they compute for the inference model.}, or by implementing doubly-reparameterized estimators \citepApp{tucker2018doubly}. 

\paragraph{SMC Samplers and MCMC.} In addition to enabling implementation of variational methods based on SMC, this work enables the interleaving of resampling and move operations to define so-called SMC samplers \citepApp{chopin2002sequential}. One recent example of are the APG samplers that we consider in our experiments \citepApp{wu2020amortized}. It is also possible to interleave importance sampling with any MCMC operator, which preserves proper weighting as long as the stationary distribution of this operator is the target density. In this space, there exists a large body of relevant work. 

\citeApp{hoffman2017learning} apply Hamiltonian Monte Carlo to samples that are generated from the encoder, which serves to improve the gradient estimate w.r.t.~the generative model, while learning the inference network using a standard reparameterized lower bound objective. \citeApp{li2017approximate} also use MCMC to improve the quality of samples from an encoder, but additionally use these samples to train the encoder by minimizing the forward KL divergence relative to the filtering distribution of the Markov chain. Since the filtering distribution after multiple MCMC steps is intractable, \citeApp{li2017approximate} use an adversarial objective to minimize the forward KL. \citet{wang2018meta} develop a meta-learning approach to learn Gibbs block conditionals. This work assumes a setup in which it is possible to sample data and latent variables from the true generative model. This approach minimizes the forward KL, but uses the learned conditionals to define an (approximate) MCMC sampler, rather than using them as proposals in a SMC sampler. 

\paragraph{Proper weighting and nested variational inference.} This work directly builds on seminal work by 
\citeApp{liu2008monte,naesseth2015nested,naesseth_elements_2019}
that formalizes proper weighting. This formalism makes it possible to reason compositionally about validity of importance samplers and corresponding variational objectives \citepApp{zimmermann2021nested}.

\subsection{Probabilistic Programming}

Probabilistic programming systems implement methods for inference in programmatically specified models. A wide variety of systems exist, which differ in the base languages they employ, the types of inference methods they provide, and their intended use cases.
One widely used approach to the design of probabilistic programming systems is to define a language in which all programs are amenable to a particular style of inference. Exemplars of this approach include Stan \citepApp{carpenter2017stan}, which emphasizes inference using Hamiltonian Monte Carlo methods in differentiable models with statically-typed support, Infer.NET \citepApp{minka2010infer} which emphasizes message passing in programs that denote factor graphs, Problog \citepApp{deraedt2007problog} and Dice \citepApp{holtzen2020scaling}, in which programs denote binary decision diagrams, and LibBi \citepApp{murray2013bayesian}, which emphasizes particle-based methods for state space models. More generally, systems that emphasize MCMC methods in programs with statically typed support fit this mold, including early systems like BUGS \citepApp{spiegelhalter1995bugs} and JAGS \citepApp{plummer2003jags}, as well as more recent systems like PyMC3 \citepApp{salvatier2016probabilistic}.

A second widely used approach to probabilistic programming is to extend a general-purpose language with functionality for probabilistic modeling, and implement inference methods that are generally applicable to programs in this language. The advantage of this design is that it becomes more straightforward to develop simulation-based models that incorporate complex deterministic functions, or to design programs that incorporate recursion and control flow. A well-known early exemplar of this style of probabilistic programming is Church \citepApp{goodman2008church}, whose modeling language is based on Scheme. Since then, many existing languages have been adapted to probabilistic programming, including Lisp variants \citepApp{mansinghka2014venture,wood2014new}, Javascript \citepApp{goodman2014design}, C \citepApp{paige2014compilation}, Scala \citepApp{pfeffer2009figaro}, Go \citepApp{tolpin2018infergo}, and Julia \citepApp{ge2018turing,cusumano-towner2019gen}. 

This second class of probabilistic programming systems is the most directly relevant to the work that we present here. A key technical consideration in the design of probabilistic programming systems is whether the modeling language is first-order or higher-order (sometimes also referred to as ``universal'') \citepApp{vandemeent2018introduction}. In first-order languages, a program denotes a density in which the support is statically determinable at compile time. Since most general-purpose languages support higher-order functions and recursion, the support of probabilistic programs in these languages is generally not statically determinable. However, the traced evaluation model that we describe in Section~\ref{sec:traced-eval} can be implemented in almost any language, and inference combinators can therefore be implemented as a DSL for inference in most of these systems.


In the context of our work, we are particularly interested in systems that extend or interoperate with deep learning systems, including Pyro~\citepApp{bingham2018pyro}, 
Edward2~\citepApp{tran2018simple}, Probabilistic Torch \citepApp{siddharth2017learning}, Scruff~\citepApp{pfeffer2018scruff}, Gen~\citepApp{cusumano-towner2019gen} and PyProb~\citepApp{baydin-2019-etalumis}. These systems provide first-class support for inference with stochastic gradient methods. While support for stochastic variational inference in probabilistic programming has been around for some time \citepApp{wingate2013automated,vandemeent2016black,ritchie2016deep}, this style of inference has become much more viable in systems where variational distributions can be parameterized using neural networks. However, to date, the methods that are implemented in these systems are typically limited to standard SVI, RWS, and IWAE objectives. We are not aware of systems that currently support stochastic variational methods that incorporate importance resampling, such as autoencoding SMC or APG samplers.

\vspace{-0.5\baselineskip}
\subsection{Inference Programming}
\vspace{-0.5\baselineskip}

The combinator-based language for inference programming that we develop in this paper builds on ideas that have been under development in the probabilistic programming community for some time.

The Venture paper described ``inference programming'' as one of the desiderata for functionality of future systems \citepApp{mansinghka2014venture}. Venture also implemented a stochastic procedure interface for manipulating traces and trace fragments, which can be understood as a low-level programming interface for inference. WebPPL \citepApp{goodman2014design} and Anglican \citepApp{tolpin2016design} define an interface for inference implementations that is based on continuation-passing-style (CPS) transformations, in which a black-box deterministic computation returns continuations to an inference backend, which implements inference computations and continues execution \citepApp{vandemeent_poplw_2016}. A number of more recent systems have implemented similar interfaces, albeit by different mechanisms than CPS transformations. Turing \citepApp{ge2018turing} uses co-routines in Julia as a mechanism for interruptible computations. Pyro~\citepApp{bingham2018pyro} and Edward2~\citepApp{tran2018simple} implement an interface in which requests to the inference backend are dispatched using composable functions that are known as ``messengers'' or ``tracers''. The arguably most general instantiation of this idea is found in PyProb \citepApp{baydin-2019-etalumis}, which employs a cross-platform Probabilistic Programming eXecution (PPX) protocol based on flatbuffers to enable computation between a program and inference backend that can be implemented in different languages \citepApp{baydin2018pyprob}. For a pedagogical discussion, see the introduction by \citeApp{vandemeent2018introduction}.

All of the above programming interfaces for inference are low-level, which is to say that it is the responsibility of the developer to unsure that quantities like importance weights and acceptance probabilities are computed in a manner that results in a correct inference algorithm. This means that users of probabilistic programming systems can in principle create their own inference algorithms, but that doing so may require considerable expertise and debugging. For this reason, recent systems have sought to develop higher-level interfaces for inference programming. Edward~\citepApp{tran2016edward} provides a degree of support for interleaving inference operations targeting different conditionals. Birch \citepApp{murray2018automated} provides constructs for structural motifs, such as state space models, which are amenable to inference optimizations. 

Exemplars of systems that more explicitly seek to enable inference programming include Gen~\citepApp{cusumano-towner2019gen} and Functional Tensors~\citepApp{obermeyer2019functional}.
Gen provides support for writing user-level code to interleave operations such as MCMC updates or MAP estimation, which can be applied to subsets of variables. For this purpose, Gen provides a generative function interface consisting of primitives \texttt{generate}, \texttt{propose}, \texttt{assess}, \texttt{update}, and \texttt{choice\_gradients}.  Recursive calls to these primitives at generative function call sites serve to compositionally implement Gen's built-in modeling language, along with hierarchical traces. \citeApp{Lew2020} considered a type system for execution traces in probabilistic programs, which allowed them to verify the correctness of certain inference algorithms at compile time.

Functional Tensors (funsors) provide an abstraction for integration that aims to unify exact and approximate inference.
Funsors providing a grammar and type system which take inspiration from modern autodifferentiation libraries.
Expressions in funsors denote discrete factors, Gaussian factors, point mass, delayed values, function application, substitution, marginalization, and plated products. Types encapsulate tensor dimensions, which permit funsors to support broadcasting.
This defines an intermediate representation that can be used for a variety of probabilistic programs and inference methods.


The inference language that we develop here differs from these above approaches in that is designed to ensure that any composition of combinators yields an importance sampler that is valid, in the sense that evaluation is properly-weighted for the unnormalized density that a program denotes. In doing so, our work takes inspiration from Hakaru~\citepApp{narayanan2016probabilistic}, which frames inference as program transformations that be composed so as to preserve a
measure-theoretic denotation \citepApp{zinkov2017composing}, as well as the work on validity of inference in higher-order probabilistic programs by \citeApp{scibior2017denotational}, which we discuss in the next Section.

\subsection{Sampling and Measure Semantics}

Programming languages have been studied in terms of two kinds of semantics: what the programs do, operational semantics, and what they mean, denotational semantics.  For probabilistic programming languages, this has usually led to a denotational \emph{measure semantics} that interpret generative model programs as measures, and an operational \emph{sampler semantics} that interpret programs as procedures for using randomness to sample from a specified distribution.  In the terms we use for our combinators library, model composition changes the target density of a program, and thus its measure semantics, while inference programming should alter the sampling semantics in a way that preserves the measure semantics.

Early work by \citetApp{Borgstrom2011,Toronto2015} characterized measure semantics for first-order probabilistic programs, with the first enabling inference by compilation to a factor graph, and the second via semi-computable preimage functions.  This demonstrates one of the simplest, but most analytically difficult, ways to perform inference in a probabilistic program: get rid of the sampling semantics and use the measure semantics to directly evaluate the relevant posterior expectation.  When made possible by program analysis, this can even take the form of symbolic disintegration of the joint distribution entailed by a generative program into observations and a posterior distribution, as in \citetApp{shan2017exact}.

\citetApp{Fong2013} gave categorical semantics to causal Bayesian networks, and via the usual compilation to a graphical model construction, to first-order probabilistic programs as well.  Further work by \citeApp{Clerc2017,Dahlqvist2018} has extended the consideration of categorical semantics for first-order PPLs.

Quasi-Borel spaces were discovered by \citeApp{Heunen2017} and quickly found to provide a good model for higher-order probabilistic programming. \citetApp{scibior2017denotational,Scibior:2018:FPM:3243631.3236778} applied these categorical semantics to prove that inference algorithms could be specified by monad transformers on sampling strategies that would preserve the categorical measure semantics over the underlying generative model.  The work of \citetApp{Scibior:2018:FPM:3243631.3236778} included an explicit consideration of importance weighting, which we have extended to cover the proper weighting of importance samples from arbitrarily nested and extended inference programs.

\vspace{\fill}

\section{Denotational Semantics of Target and Inference Programs} 
\label{apx:denotational_semantics}

To reason about validity of inference, we need to specify what density a target program \pt{p} or inference program \pt{q} denotes. We begin with target programs, which have the grammar \pt{p $::=$ f $\mid$ extend(p, f)}. For a primitive program $\mpt{f}$, the denotational semantics inherit trivially from the axiomatic denotational semantics of the modeling language. For a program \pt{extend(p, f)} we define the density as the composition of densities of the inputs
\[
    \begin{prooftree}
        \hypo{
            \begin{matrix}
             c_1, \tau_1, \rho_1, w_1 \eval \mpt{p($c_0$)}
             \qquad
             c_2, \tau_2, \rho_2, w_2 \eval \mpt{f($c_1$)}\\[4pt]
             \text{dom}(\rho_1) \cap \text{dom}(\rho_2) = \emptyset
             \qquad 
             \text{dom}(\rho_2) = \text{dom}(\tau_2)
             \qquad
             \tau_3 = \tau_2 \oplus \tau_1
             \end{matrix}
             }
        \infer1{\den{\mpt{extend(p, f)($c_0$)}}(\tau_3) 
                =
                \den{\mpt{f($c_1$)}}(\tau_2) \:
                \den{\mpt{p($c_0$)}}(\tau_1)
                }
    \end{prooftree}
\]
In this rule, we omit subscripts $\den{\cdot}_\gamma$ and $\den{\cdot}_p$, since this rule applies to both prior and the unnormalized density. As in the case of primitive programs, we here adopt a convention in which the support is implicitly defined as the set of traces $\tau_3 = \tau_2 \oplus \tau_1$ by combining disjoint traces $\tau_1$ and $\tau_2$ that can be generated by evaluating the composition of \pt{p} and \pt{f}.

Inference programs have a grammar \pt{q $::=$ p $\mid$ compose(q$'$, q) $\mid$ resample(q) $\mid$ propose(p, q)}. For each expression form, we define the density that a program denotes in terms of the density of its corresponding target program. To do so, we define a program tranformation \pt{target(q)}. This transformation replaces all sub-expressions of the form \pt{propose(p, q)} with  their targets \pt{p} and all sub-expressions \pt{resample(q)} with \pt{q}. We define the transformation recursively
\[
    \begin{prooftree}
        \hypo{
            \phantom{\mpt{q}'_1 = \mpt{target(q$_1$)}}
        }
        \infer1{
          \mpt{p} = \mpt{target(p)} \phantom{'}
        }
    \end{prooftree}
    \quad
    \begin{prooftree}
        \hypo{
            \phantom{\mpt{q}'_1 = \mpt{target(q$_1$)}}
        }
        \infer1{
          \mpt{p} = \mpt{target(propose(p,q))} \phantom{'}
        }
    \end{prooftree}
    \quad
    \begin{prooftree}
        \hypo{
            \mpt{q}' = \mpt{target(q)}
        }
        \infer1{
          \mpt{q}' = \mpt{target(resample(q))} \phantom{'}
        }
    \end{prooftree}
    \quad
    \begin{prooftree}
        \hypo{
          \mpt{q}'_1 = \mpt{target(q$_1$)} 
          \qquad
          \mpt{q}'_2 = \mpt{target(q$_2$)} 
        }
        \infer1{
          \mpt{compose(q$^\prime_2$,q$^\prime_1$)} 
          = 
          \mpt{target(compose(q$_2$,q$_1$))}
        }
    \end{prooftree}
\]
Since this transformation removes all instances of propose and resample forms, transformed programs \pt{q$'=$target(q)} define a simplified grammar \pt{q$' ::=$ p $\mid$ compose(q$'_1$, q$'_2$)}. We now define the denotational semantics for inference programs as
\[    
    \begin{prooftree}
        \hypo{
          \mpt{q}' = \mpt{target(q)}
        }
        \infer1{
        \begin{matrix}
            \den{\mpt{q($c$)}}
            =
            \den{\mpt{q$^\prime$($c$)}}
        \end{matrix}
        }
    \end{prooftree}
\]
The denotational semantics for a transformed program are trivially inherited from those for target programs when \pt{q}$'=$\pt{p}, whereas the denotational semantics for a composition \pt{compose(q$'_2$, q$'_1$)} are analogous to those of an extension \pt{extend(p, f)}
\[    
    \begin{prooftree}
        \hypo{
            \begin{matrix}
             c_1, \tau_1, \rho_1, w_1 \eval \mpt{q$^\prime_1$($c_0$)}
             \quad
             c_2, \tau_2, \rho_2, w_2 \eval \mpt{q$^\prime_2$($c_1$)}\\[4pt]
             \text{dom}(\rho_1) \cap \text{dom}(\rho_2) = \emptyset
             \quad 
             \tau_3 = \tau_1 \oplus \tau_2
             \end{matrix}
             }
        \infer1{\den{\mpt{compose(q$^\prime_2$, q$^\prime_1$)($c_0$)}}(\tau_3) 
                =
                \den{\mpt{q$^\prime_2$($c_1$)}}(\tau_2) \:
                \den{\mpt{q$^\prime_1$($c_0$)}}(\tau_1)
                }
    \end{prooftree}
\]

\section{Evaluation under Substitution} 
\label{apx:substitution}

To use an extended program as a target for a proposal, we need to define its evaluation under substitution. We define the operational semantics of this evaluation in a manner that is analogous to the unconditioned case
\[
\begin{prooftree}
  \hypo{
    \begin{matrix}
      c_1, \tau_1, \rho_1, w_1 \ceval{\tau_0}{\mpt{p($c_0$)}} \qquad
      c_2, \tau_2, \rho_2, w_2 \ceval{\tau_0}{\mpt{f($c_1$)}}
      \\[4pt]
      \text{dom}(\rho_1) \cap \text{dom}(\rho_2) = \emptyset 
      \qquad
      \text{dom}(\rho_2) = \text{dom}(\tau_2)
    \end{matrix}
  }
  \infer1{c_1, \tau_1 \oplus \tau_2, \rho_1 \oplus \rho_2, w_1 \cdot w_2
          ~\ceval{\tau_0}{\mpt{extend(p, f)($c_0$)}}}
\end{prooftree}
\]
This rule, like other rules, defines a recursion. In this case, the recursion ensures that we can perform conditioned evaluation for any target program $\mpt{p}$.

We define evaluation under substitution of an inference program \pt{q} by performing an evaluation under substitution for the corresponding target program 
\[
\begin{prooftree}
  \hypo{
    \begin{matrix}
        c,\tau,\rho,w \eval \mpt{target(q)($c_0$)}[\tau_0]
    \end{matrix}
  }
  \infer1{c,\tau,\rho,w \eval \mpt{q($c_0$)}[\tau_0]}
\end{prooftree}
\]
As in the previous section, the transformed programs define a grammar \pt{q$' ::=$ p $\mid$ compose(q$'_1$, q$'_2$)}. In the base case $\mpt{q}' = \mpt{p}$, evaluation under substitution is defined as above. Evaluation under substitution of a program \pt{compose(q$'_1$, q$'_2$)} is once again defined by recursively evaluating inputs under substitution, as with the extend combinator
\[
\begin{prooftree}
  \hypo{
    \begin{matrix}
      c_1, \tau_1, \rho_1, w_1 \ceval{\tau_0}{\mpt{q$^\prime_1$($c_0$)}} \qquad
      c_2, \tau_2, \rho_2, w_2 \ceval{\tau_0}{\mpt{q$^\prime_2$($c_1$)}} \qquad
      \text{dom}(\rho_1) \cap \text{dom}(\rho_2) = \emptyset 
    \end{matrix}
  }
  \infer1{c_1, \tau_1 \oplus \tau_2, \rho_1 \oplus \rho_2, w_1 \cdot w_2
          ~\ceval{\tau_0}{\mpt{compose(q$^\prime_2$, q$^\prime_1$)($c_0$)}}}
\end{prooftree}
\]
Note that the operational semantics do not rely on evaluation under substitution of inference programs \pt{q}, since conditional evaluation is only every performed for target programs. However, the proofs in Section~\ref{apx:proper_weighting_proofs} do make use of the definition of the prior under substitution, which is identical to the definition for primitive programs
\[
\begin{prooftree}
  \hypo{
    c_1, \tau_1, \rho_1, w_1 \ceval{\tau_0}{\mpt{q($c_0$)}} 
  }
  \infer1{
    \displaystyle
    \den{\mpt{p($c_0$)}[\tau_0]}_p(\tau_1)
    ~=~
    p_{q[\tau_0]}(\tau_1 ; c_0)
    ~= 
    \!\!\!\!
    \prod_{\alpha \in \text{dom}(\tau_1) \setminus \text{dom}(\tau_0)}
    \!\!\!\!
    \rho_1(\alpha)
  }
\end{prooftree}
\]

\section{Evaluation in Context}
\label{apx:objective_computation}
To accumulating loss in the combinators framework, we reframe our operational semantics so that they
are evaluated in a context of a user-defined objective function $\ell : (\rho_q, \rho_p, w, v) \to \mathbb{R}$
and an accumulated loss $\mathcal{L}$ which is initialized to $0 \in \mathbb{R}$.
To denote an expression $\cdot$ evaluated in the context of $\mathcal{L}$ and $\ell$, we describe the notation
$$
\langle \mathcal{L}, \ell, \cdot\rangle.
$$

To examine traced evaluation $c, \tau, \rho, w  \eval \mpt{f}(c')$, in context we explicitly denotes both the input and output context,

\[
\langle \mathcal{L}, \ell, (c, \tau, \rho, w) \rangle \eval \langle \mathcal{L}, \ell, \mpt{f}(c') \rangle.
\]

\newcommand{\wrinkle}[2]{\langle #1 \rangle \Lsquigarrow{\phantom{in}} \langle #2 \rangle}
Note that $\ell$ is initialized by the user and does not change during program execution.
To account for this and lighten the syntax, we define syntactic sugar of:
\[
\wrinkle{\mathcal{L}, (c, \tau, \rho, w)}{\mathcal{L}', \mpt{f}(c')}
\]

When reframing combinators to accumulate loss, the operational semantics of \pt{compose}, \pt{extend}, and \pt{resample} only thread their input $\mathcal{L}$ through program execution.
For these combinators, we simply exchange the traced evaluation of $\eval$ with our syntactic sugar of $\Lsquigarrow{\phantom{in}}$ as seen in Figure~\ref{fig:loss-semantics}.

\pt{propose} is the only combinator in which loss is accumulated
\[
\begin{prooftree}
  \hypo{
    \begin{matrix}
      \mathcal{L}_2 = \ell (\rho_1, \rho_2, w_1, w_2/ u_1) + \mathcal{L}_1  \\
      \wrinkle{\mathcal{L}_1, (c_1, \tau_1, \rho_1, w_1)}{\mathcal{L}_0, \mpt{q($c_0$)}} \quad
      c_2, \tau_2, \rho_2, w_2 ~\eval{}~ \mpt{p($c_0$)}[\tau_1] \\
      c_3, \tau_3, \rho_3, w_3 \: \ceval{\tau_2}{\mpt{marginal(p)($c_0$)}} \\
      \displaystyle
      u_1 = \prod_{\alpha \in
                          \text{dom}(\rho_1)
                             \setminus (\text{dom}(\tau_1)
                                        \setminus \text{dom}(\tau_2))}
                 \rho_1(\alpha)
    \end{matrix}
  }
  \infer1{\langle \mathcal{L}_2, \ell, (c_3, \tau_3, \rho_3,~ w_2 \cdot w_1 / u_1) \rangle 
          ~\eval{}~ \langle \mathcal{L}_0, \ell, \mpt{propose(p, q)($c_0$)} \rangle}
\end{prooftree}.
\]



\begin{table}[!t]
     \resizebox{0.67\textwidth}{!}{\begin{minipage}{\textwidth}
     \centering
     \begin{tabular}{c}
       \begin{tabular}{cc}
       {
       \begin{prooftree}
         \hypo{
           \begin{matrix}
             \wrinkle{\mathcal{L}_1, (c_1, \tau_1, \rho_1, w_1)}{\mathcal{L}_0, \mpt{q$_1$($c_0$)}} \quad
             \wrinkle{\mathcal{L}_2, (c_2, \tau_2, \rho_2, w_2)}{\mathcal{L}_1, \mpt{q$_2$($c_1$)}} \\[3pt]
             \text{dom}(\rho_1) \cap \text{dom}(\rho_2) = \emptyset
           \end{matrix}
         }
         \infer1{
             \wrinkle{\mathcal{L}_2, (c_2, \tau_2 \oplus \tau_1, \rho_2 \oplus \rho_1,~ w_2 \cdot w_1)}{\mathcal{L}_0, \mpt{compose(q$_2$, q$_1$)($c_0$)}}}
          
       \end{prooftree}
       } & {
       \begin{prooftree}
         \hypo{
           \begin{matrix}
             \wrinkle{\mathcal{L}_1, (c_1, \tau_1, \rho_1, w_1)}{\mathcal{L}_0, \mpt{p($c_0$)}} \quad
             \wrinkle{\mathcal{L}_2, (c_2, \tau_2, \rho_2, w_2)}{\mathcal{L}_1, \mpt{f($c_1$)}} \\[3pt]
             \text{dom}(\rho_1) \cap \text{dom}(\rho_2) = \emptyset 
             \qquad
             \text{dom}(\rho_2) = \text{dom}(\tau_2)
           \end{matrix}
         }
         \infer1{
             \wrinkle{\mathcal{L}_2, (c_2, \tau_1 \oplus \tau_2, \rho_1 \oplus \rho_2, w_1 \cdot w_2)}{\mathcal{L}_0,
             \mpt{extend(p, f)($c_0$)}}
                 }
       \end{prooftree}
       } 
       \end{tabular}
       \\[30pt]
       {
       \begin{prooftree}
         \hypo{
           \begin{matrix}
             \wrinkle{\mathcal{L}_1, (\v{c}_1, \v{\tau}_1, \v{\rho}_1, \v{w}_1 )}{\mathcal{L}_0, \mpt{q($\v{c}_0$)}} \quad
             \v{a}_1 \sim \textsc{resample}(\v{w}_1) \\[3pt]
             \v{c}_2, \v{\tau}_2, \v{\rho}_2 = \textsc{reindex}(\v{a}_1, \v{c}_1, \v{\tau}_1, \v{\rho}_1) \qquad
             \v{w}_2 = \textsc{mean}(\v{w}_1)
           \end{matrix}
         }
         \infer1{
             \wrinkle{\mathcal{L}_1, (\v{c}_2, \v{\tau}_2, \v{\rho}_2, \v{w}_2)}{\mathcal{L}_0, \mpt{resample(q)($\v{c}_0$)}}}
       \end{prooftree}
       } \\
    \end{tabular}
    \end{minipage}}
    \caption{Operational semantics for evaluating \pt{compose}, \pt{extend}, and \pt{resample} combinators in the context of a loss function $\ell$ and accumulated loss $\mathcal{L}$.}
    \label{fig:loss-semantics}
\end{table}

\section{Proper weighting of programs}

\label{apx:proper_weighting_proofs}
\begin{lemma}[Strict proper weighting of the extend combinator]
\label{lem:spw_extend}
Evaluation of a target program $\mpt{p}_2 = \mpt{extend(p$_1$, f)}$ is strictly properly weighted for its unnormalized density $\den{\mpt{p}_2}_\gamma$ when evaluation of \pt{p}$_1$ is strictly properly weighted for $\den{\mpt{p}_1}_\gamma$.
\end{lemma}
\begin{proof}
Recall from Section~\ref{subsec:operational}, that the program \pt{p$_{2}$} denotes a composition 
\begin{equation}
\begin{prooftree}
    \hypo{
      c_{1}, \tau_{1}, \rho_{1}, w_{1} \eval \mpt{p$_{1}$(c$_0$)}
      \quad
      c_f, \tau_f, \rho_f, w_f \eval \mpt{f(c$_1$)}
      \quad
      \tau_2 = \tau_{1} \oplus \tau_f
      \quad
      \rho_2 = \rho_{1} \oplus \rho_f
    }
    \infer1{
      c_{1}, \tau_{2}, \rho_{2}, w_{1} \eval \mpt{p$_2$($c_0$)}
      \qquad
      \den{\mpt{p}_{2}}_\gamma(\tau_{2} ; c_0) 
      = \den{\mpt{p}_{1}}_\gamma(\tau_{1} ; c_0)
        \cdot
        \den{\mpt{f}}_\gamma(\tau_{f} ; c_1)
    }
\end{prooftree}
\end{equation}
Our induction hypothesis is that \pt{p$_{1}$} is strictly properly weighted for its density $\den{\mpt{p}_{1}}_\gamma := \gamma_{1}$. Since the primitive program \pt{f} may only include unobserved variables, $w_f=1$ and its evaluation is properly weighted relative to the prior density $\den{\mpt{f}}_\gamma = \den{\mpt{f}}_p := p_f$. Strict proper weighting with respect to $\den{\mpt{p$_2$}}_\gamma = \gamma_2$ follows directly from definitions
\begin{align}
    \begin{split}
    \mathbb{E}_{\mpt{p$_2$($c_0$)}}
    \Big[
        w_1
        \:
        h(\tau_2)
    \Big]
    &=
    \mathbb{E}_{\mpt{p$_1$($c_0$)}}
    \Big[
    w_1
    \:
    \mathbb{E}_{\mpt{f($c_1$)}}
    \big[
        h(\tau_1 \oplus \tau_f)
    \big]
    \Big],
    \\
    &=
    \int
    \!\! d\tau_1 \:
    \gamma_1(\tau_1 ; c_0)
    \int
    \!\! d\tau_f \:
    p_f(\tau_f ; c_1) \: h(\tau_1 \oplus \tau_f),
    \\
    &=
    \int
    \!\! d\tau_1 d\tau_f \:
    \gamma_2(\tau_1 \oplus \tau_f ; c_0) \:
    h(\tau_1 \oplus \tau_f)
    \\
    &=
    \int
    \!\! d\tau_2 \:
    \gamma_2(\tau_2 ; c_0) \:
    h(\tau_2)
    .
    \end{split}
\end{align}
Note in particular that $Z_1(c_0) = Z_2(c_0)$, since the normalizing constant $Z_f(c_1) = 1$ for the program \pt{f}.

\end{proof}


\begin{theorem}[Strict proper weighting of target programs]
\label{thm:target_proper_weighting} Evaluation of a target program \pt{p} is (strictly) properly weighted for the unnormalized density $\den{\mpt{p}}_\gamma$ that it denotes.
\end{theorem}
\begin{proof}
By induction on the grammar $\mpt{p} ::= \mpt{f} \mid \mpt{extend(p, f)}$. 
\begin{itemize}
    \item \emph{Base case}: $\mpt{p} = \mpt{f}$. This follows from Proposition~\ref{prop:primitive_proper_weighting}.
    \item \emph{Inductive case}: $\mpt{p}_2 = \mpt{extend(p$_1$, f)}$. This follows from Lemma~\ref{lem:spw_extend}.
\end{itemize}
\end{proof}

\begin{lemma}[Strict proper weighting of the resample combinator]
\label{lem:resample_proper_weighting}
Evaluation of a program $\mpt{q}_2 = \mpt{resample(q$_1$)}$ is strictly properly weighted for its unnormalized density $\den{\mpt{q}_2}_\gamma$ when evaluation of \mpt{q}$_1$ is strictly properly weighted for $\den{\mpt{q}_1}_\gamma$.
\end{lemma}
\begin{proof}
In Section~\ref{subsec:operational}, we defined the operational semantics for the \mpt{resample} combinator as
\[
\begin{prooftree}
  \hypo{
    \begin{matrix}
      \v{c}_1, \v{\tau}_1, \v{\rho}_1, \v{w}_1 \eval \mpt{q}_1(\v{c}_0) \qquad
      \v{a}_1 \sim \textsc{resample}(\v{w}_1) \qquad
      \v{c}_2, \v{\tau}_2, \v{\rho}_2 = \textsc{reindex}(\v{a}_1, \v{c}_1, \v{\tau}_1, \v{\rho}_1) \qquad
      \v{w}_2 = \textsc{mean}(\v{w}_1)
    \end{matrix}
  }
  \infer1{\v{c}_2, \v{\tau}_2, \v{\rho}_2, \v{w}_2
          ~\eval{}~ \mpt{resample(q$_1$)}(\v{c}_0)}
\end{prooftree}.
\]
Whereas other combinators act on samples individually, the resample combinator accepts and returns a collection of samples. Bold notation signifies tensorized objects. For notational simplicity, we assume in this proof that all objects contain a single dimension, e.g.~$\v{\tau} = [\tau^1, \dots, \tau^L]$, which is also the dimension along which resampling is performed, but this is not a requirement in the underlying implementation.

Let $\den{\mpt{q}_2(c_0)}_\gamma = \den{\mpt{q}_1(c_0)}_\gamma = \gamma(\cdot \:; c_0)$ denote the unnormalized density of the program. Our goal is to demonstrate that outgoing samples are individually strictly properly weighted for the unnormalized density $\gamma(\cdot \:; c_0)$,
\begin{align}
    \mathbb{E}_{\mpt{q$_2$($c_0$)}} 
    \left[ 
        w^l_2 \: h(\tau^l_2) 
    \right] 
    &= 
    \int 
    d \tau'
    \:
    \gamma(\tau' ; c_0)
    \:
    h(\tau'),
\end{align}
under the inductive hypothesis that incoming samples are strictly properly weighted,
\begin{align}
    \mathbb{E}_{\mpt{q$_1$($c_0$)}} 
    \left[ 
        w^l_1 \: h(\tau^l_1) 
    \right] 
    &= 
    \int 
    d \tau'
    \:
    \gamma(\tau' ; c_0)
    \:
    h(\tau').
\end{align}
The resample combinator randomly selects ancestor indices $\v{a}_1 \sim \textsc{resample}(\v{w}_1)$. Informally, this procedure selects $a^l_1 = k$ with probability proportional to $w^k_1$. More formally, this procedure must satisfy   
\begin{align}
    \mathbb{E}_{\textsc{resample}(\v{w}_1)}\left[\mathbb{I}[a^{l} = k]\right] = \frac{w^k_1}{\sum_{l} w^l_1}.
\end{align}
The outgoing return value, trace, and weight, are then reindexed according to $\v{a}_1$, whereas the outgoing weights are set to the average of the incoming weights
\begin{align}
    c^l_2 &=  c_1^{a_1^l}, &
    \tau^l_2 &= \tau^{a_1^l}_1, &
    \rho^l_2 &= \rho^{a_1^l}_1, &
    w^l_2 &= \frac{1}{L} \sum_{l'=1}^L w_1^{l'}.
\end{align}
Strict proper weighting now follows directly from definitions
\begin{align}
  \begin{split}
    \mathbb{E}_{\mpt{q}_2(c_0)} \left[ w^l_2 \: h(\tau^l_2) \right] 
    &= \mathbb{E}_{\mpt{q}_2(c_0)} \left[ 
         \Big(\frac{1}{L} \sum_{l'} w^{l'}_1 \Big) \: h(\tau_1^{a^l_1}) 
      \right] 
    \\
    &
    = \mathbb{E}_{\mpt{q}_1(c_0)} \left[ 
         \Big( \frac{1}{L} \sum_{l'} w^{l'}_1 \Big) \:
         \mathbb{E}_{\textsc{resample}(\v{w}_1)} \left[ 
            h(\tau_1^{a^l_1}) 
         \right]  
      \right] 
    \\
    &
    = \mathbb{E}_{\mpt{q}_1(c_0)} \left[ 
         \Big( \frac{1}{L} \sum_{l'} w^{l'}_1 \Big) \:
         \mathbb{E}_{\textsc{resample}(\v{w}_1)} \left[ 
            \sum_{k}
            \mathbb{I}[a_1^l = k] \:
            h(\tau_1^{k}) 
         \right]  
      \right] 
    \\
    &
    = \mathbb{E}_{\mpt{q}_1(c_0)} \left[ 
         \Big( \frac{1}{L} \sum_{l'} w^{l'}_1 \Big) \:
         \sum_k \:
         \mathbb{E}_{\textsc{resample}(\v{w}_1)} \left[ 
            \mathbb{I}[a_1^l = k]
         \right]  
         h(\tau_1^{k}) 
      \right] 
    \\
    &
    = \mathbb{E}_{\mpt{q}_1(c_0)} \left[ 
         \Big( \frac{1}{L} \cancel{\sum_{l'} w^{l'}_1} \Big) \:
         \sum_k
         \frac{w_1^k}{\cancel{\sum_{l''} w_1^{l''}}}
         \:
         h(\tau_1^{k}) 
      \right] 
    \\
    &
    = \mathbb{E}_{\mpt{q}_1(c_0)} \left[ 
         \frac{1}{L}
         \sum_k
         w_1^k
         \:
         h(\tau_1^{k}) 
      \right]
    =
    \frac{1}{L}
    \sum_k
    \:
    \mathbb{E}_{\mpt{q}_1(c_0)} \left[ 
      w_1^k
      \:
      h(\tau_1^{k})
    \right]
    .
  \end{split}
\end{align}
\end{proof}

\begin{lemma}[Strict proper weighting of the compose combinator]
\label{lem:compose_proper_weighting}
Evaluation of the program $\mpt{q}_3 = \mpt{compose(q$_1$, q$_2$)}$ is strictly properly weighted for its unnormalized density $\den{\mpt{compose(q$_1$, q$_2$)}}_\gamma$ when evaluation of \pt{q$_1$} and \pt{q$_2$} is strictly properly weighted for the unnormalized densities $\den{\mpt{q$_1$}}_\gamma$ and $\den{\mpt{q$_2$}}_\gamma$.
\end{lemma}
\begin{proof} In Section~\ref{subsec:operational}, we defined the operational semantics for the compose combinator as
\[
\begin{prooftree}
  \hypo{
    \begin{matrix}
      c_1, \tau_1, \rho_1, w_1 \eval \mpt{q$_1$($c_0$)} \quad
      c_2, \tau_2, \rho_2, w_2 \eval \mpt{q$_2$($c_1$)} \quad
      \text{dom}(\rho_1) \cap \text{dom}(\rho_2) = \emptyset \quad
      \tau_3 = \tau_2 \oplus \tau_1 \quad
      \rho_3 = \rho_2 \oplus \rho_1 \quad
      w_3 = w_2 \cdot w_1 
    \end{matrix}
  }
  \infer1{c_2, \tau_3, \rho_3, w_3
          ~\eval{}~ \mpt{compose(q$_2$, q$_1$)($c_0$)}}
\end{prooftree},
\]
and its denotation as the product of conditional densities
\[
    \begin{prooftree}
        \hypo{
            \begin{matrix}
             c_1, \tau_1, \rho_1, w_1 \eval \mpt{q$_1$($c_0$)}
             \quad
             c_2, \tau_2, \rho_2, w_2 \eval \mpt{q$_2$($c_1$)}
             \qquad
             \text{dom}(\rho_1) \cap \text{dom}(\rho_2) = \emptyset
             \end{matrix}
             }
        \infer1{\den{\mpt{compose(q$_2$, q$_1$)($c_0$)}}(\tau_1 \oplus \tau_2) 
                =
                \den{\mpt{q$_2$($c_1$)}}(\tau_2) \:
                \den{\mpt{q$_1$($c_0$)}}(\tau_1)
                }
    \end{prooftree}.
\]

Let $\den{\mpt{q}_3}_\gamma = \gamma_3$ denote the unnormalized density of the composition. We will show that, for any measurable function $h(\tau_3)$,
\begin{align*}
    \mathbb{E}_{\mpt{q$_3$($c_0$)}} \left[ 
      w_3 \: h(\tau_3)
    \right]
    =
    Z_3(c_0)
    \int 
    d\tau_3 
    \:
    \gamma_3(\tau_3 ; c_0) 
    \:
    h(\tau_3).
\end{align*}
We can express the expectation with respect to the program \pt{q}$_3$ as a nested expectation with respect to \pt{q}$_1$ and \pt{q}$_2$
\begin{align*}
    \mathbb{E}_{\mpt{q$_3$($c_0$)}} \left[ 
      w_3 \: h(\tau_3)
    \right]
    &=
    \mathbb{E}_{\mpt{q$_1$($c_0$)}} \left[ 
      w_1 \: \mathbb{E}_{\mpt{q$_2$($c_1$)}} \left[ 
        w_2 \: h(\tau_2 \oplus \tau_1) 
      \right] 
    \right],
\end{align*}
Let $\den{\mpt{q}_1}_\gamma = \gamma_1$, $\den{\mpt{q}_2}_\gamma = \gamma_2$ of the inputs and their composition.  By the induction hypothesis, we can express rewrite both expectations as integrals with respect to $\gamma_1$ and $\gamma_2$. Strict proper weighting follows from definitions
\begin{align*}
    \mathbb{E}_{\mpt{q$_3$($c_0$)}} \left[ 
      w_3 \: h(\tau_3)
    \right]
    &=
    \int d\tau_1 \:
    \gamma_1(\tau_1 ; c_0)
    \int d\tau_2 ~
    \gamma_2(\tau_2 ; c_1)
    h(\tau_1 \oplus \tau_2 ; c_0)
    \\
    &=
    \int d\tau_1 
    \int d\tau_2 ~
    \gamma_1(\tau_1 ; c_0) \:
    \gamma_2(\tau_2 ; c_1) \:
    h(\tau_1 \oplus \tau_2 ; c_0)
    \\
    &=
    \int d\tau_3 \:
    \gamma_3(\tau_3 ; c_0) \: 
    h(\tau_3 ; c_0).
\end{align*}
\end{proof}

\begin{lemma}[Strict proper weighting of the propose combinator]
\label{lem:propose_proper_weighting}
Evaluation of a program \pt{propose(p, q)} is strictly properly weighted for the unnormalized density $\den{\mpt{propose(p, q)}}_\gamma$ when evaluation of \pt{q} is strictly properly weighted for the unnormalized density $\den{\mpt{q}}_\gamma$.
\end{lemma}
\begin{proof}
Recall from Section~\ref{subsec:operational} that that the operational semantics for propose are
\[
\begin{prooftree}
  \hypo{
    \begin{matrix}
      c_1, \tau_1, \rho_1, w_1 \eval \mpt{q}(c_0) \qquad
      c_2, \tau_2, \rho_2, w_2 \: \ceval{\tau_1}{\mpt{p($c_0$)}} \qquad 
      c_3, \tau_3, \rho_3, w_3 \: \ceval{\tau_2}{\mpt{marginal(p)($c_0$)}} \\[6pt]
      \displaystyle
      u_1 = \prod_{\alpha \in 
                          \text{dom}(\rho_{1}) 
                             \setminus (\text{dom}(\tau_{1}) 
                                        \setminus \text{dom}(\tau_{2}))}
                 \rho_1(\alpha)
    \end{matrix}
  }
  \infer1{c_3, \tau_3, \rho_3,~
          w_2 \cdot w_1 / u_1
          ~\eval{}~ \mpt{propose(p, q)}(c_0)}
\end{prooftree}
\]
Our aim is to demonstrate that evaluation of \pt{propose(p, q)} is strictly properly weighted for
\begin{align*}
    \den{\mpt{propose(p, q)(c$_0$)}}_\gamma 
    = 
    \den{\mpt{marginal(p)(c$_0$)}}_\gamma
    = 
    \gamma_p(\cdot ; c_0).
\end{align*}
This is to say that, for any measurable $h(\tau_3)$
\begin{align*}
    \mathbb{E}_{\mpt{q(c$_0$)}} \left[ 
      \frac{w_2 w_1}{u_1} \: 
      h(\tau_3) 
    \right] 
    &=
    \int d\tau_3 \:
    \gamma_p(\cdot ; c_0) \:
    h(\tau_3).
\end{align*}
We start by expressing the expectation with respect to \pt{q$_2$($c_0$)} as an expectation with respect to \pt{q$_1$($c_0$)} and \pt{p($c_0$)}$[\tau_1]$, and use the inductive hypothesis to express the first expectation as an integral with respect to $\den{\mpt{q($c_0$)}}_\gamma=\gamma_q(\cdot ; c_0)$,
\begin{align*}
    \mathbb{E}_{\mpt{q$_2$($c_0$)}} \left[
      \frac{w_2 w_1}{u_1} \: 
      h(\tau_3) 
    \right] 
    &= 
    \mathbb{E}_{\mpt{q$_1$($c_0$)}} \left[
      w_1 \: \mathbb{E}_{\mpt{p($c_0$)}[\tau_1]} \left[ 
        \frac{w_2}{u_1} \: h(\tau_3) 
      \right] 
    \right]\\
    &= 
    \int 
    d\tau_1 \:
    \gamma_q(\tau_1 ; c_0) \:
    \mathbb{E}_{\mpt{p($c_0$)}[\tau_1]} \left[ 
      \frac{w_2}{u_1} \: h(\tau_3) 
    \right].
\end{align*}
We use $\den{\mpt{p($c_0$)}} = \tilde{\gamma}_p(\cdot \:; c_0)$ to refer to the density that \pt{p} denotes, which possibly extends the density $\gamma_p(\cdot \: ; c_0)$ using one or more primitive programs. We then use Equation~\ref{eq:propose-weight} to replace $w_2/u_1$ with the relevant extended-space densities, and simplify the resulting expressions
\begin{align*}
    \int d\tau_1 \: 
    \gamma_{1}(\cdot; c_0) \:
    \mathbb{E}_{\mpt{p($c_0$)}[\tau_1]} \left[ 
      \frac{w_2}{u_1} h(\tau_3) 
    \right] 
    &= 
    \int d\tau_1 \: 
    \gamma_{q}(\tau_1; c_0) \: 
    \mathbb{E}_{\mpt{p($c_0$)}[\tau_1]} \left[ 
      \frac{\tilde{\gamma}_{p}(\tau_2; c_0) p_{q[\tau_2]}(\tau_1; c_0)}
          {\gamma_{q}(\tau_1; c_0) p_{p[\tau_1]}(\tau_2; c_0)} h(\tau_3) \right] \\
    &= 
    \int d\tau_1 \: 
    \frac{\cancel{\gamma_{q}(\tau_1; c_0)}}
         {\cancel{\gamma_{q}(\tau_1; c_0)}} \: 
    \mathbb{E}_{\mpt{p($c_0$)}[\tau_1]} \left[ 
      \frac{\tilde{\gamma}_{p}(\tau_2; c_0) \: p_{q[\tau_2]}(\tau_1; c_0)}
          {p_{p[\tau_1]}(\tau_2; c_0)} \:
      h(\tau_3) 
    \right]\\
    &= 
    \int d\tau_1 \:
    \int d\tau_2 \: 
    p_{p[\tau_1]}(\tau_2; c_0) \: 
    \frac{\tilde{\gamma}_{p}(\tau_2; c_0) \: p_{q [\tau_2]}(\tau_1; c_0)}
         {p_{p[\tau_1]}(\tau_2; c_0)} h(\tau_3) \\
    &= 
    \int d\tau_1 \: 
    \int d\tau_2 \: 
    \frac{\cancel{p_{p[\tau_1]}(\tau_2; c_0)}}
         {\cancel{p_{p[\tau_1]}(\tau_2; c_0)}} \:
    \tilde{\gamma}_{p}(\tau_2; c_0) \: p_{q [\tau_2]}(\tau_1; c_0) \: h(\tau_3) \\
    &= 
    \int d\tau_2 \: \tilde{\gamma}_{p}(\tau_2; c_0) \: h(\tau_3) \: 
    \cancel{\int d\tau_1 p_{q [\tau_2]}(\tau_1; c_0)}, \\
    &= 
    \int d\tau_3 \: \gamma_{p}(\tau_3; c_0) \: h(\tau_3).
\end{align*}
In the final equality, we rely on the fact that $\gamma_p(\tau_3 ; c_0)$ is the marginal of $\tilde{\gamma}_p(\tau_2 ; c_0)$ with respect to the set of auxiliary variables $\text{dom}(\tau_2) \setminus \text{dom}(\tau_3)$.

\end{proof}

\begin{theorem}[Strict proper weighting of compound inference programs]
\label{thm:compound_program_proper_weighting}
Compound inference programs \mpt{q} are strictly properly weighted for their unnormalized densities $\gamma_q$.
\end{theorem}
\begin{proof}
By induction on the grammar for \mpt{q}.
\begin{itemize}
    \item \emph{Base case:} $\mpt{q} = \mpt{p}$.  Theorem~\ref{thm:target_proper_weighting} above provides a proof.
    \item \emph{Inductive case:} $\mpt{q}_2 = \mpt{resample(q$_1$)}$. This follows from Lemma~\ref{lem:resample_proper_weighting}.
    \item \emph{Inductive case:} $\mpt{q}_3 = \mpt{compose(q$_1$, q$_2$)}$.  This follows from Lemma~\ref{lem:compose_proper_weighting}.
    \item \emph{Inductive case:} $\mpt{q}_2 = \mpt{propose(p, q$_1$)}$.  This follows from Lemma~\ref{lem:propose_proper_weighting}.
\end{itemize}
\end{proof}

\section{Gradient computations}
\label{apx:gradient_computation}
\subsection{Stochastic Variational Inference (SVI).} 

Let \pt{q$_2$ $=$ propose(p, q$_1$)} be a program in which the initial inference program \mpt{q$_1$}, target program \mpt{p}, and inference program \mpt{q$_2$} denote the densities 
\begin{align*}
    \den{\mpt{q$_1$(c$_0$)}}_\gamma 
    &= \gamma_{q}(\cdot \:; c_0, \phi)
    &
    \den{\mpt{p(c$_0$)}}_\gamma 
    &= 
    \tilde \gamma_p(\cdot \:; c_0, \theta),
    &
    \den{\mpt{q$_2$(c$_0$)}}_\gamma 
    &= 
    \gamma_p(\cdot \:; c_0, \theta),
    ,
\end{align*}
with parameters $\theta$ and $\phi$ respectively.
Notice that the target program \mpt{p} and the inference program \mpt{q$_2$} denote the same density $\den{\mpt{p(c$_0$)}}_\gamma = \den{\mpt{q$_2$(c$_0$)}}_\gamma$ 
and hence, as a result of Theorem \ref{thm:target_proper_weighting}, the evaluation of \mpt{q$_2$} is strictly properly weighted for $\gamma_p$. 
Applying definition \ref{def:proper_weighting} we can now write \begin{align*}
    Z_{p}(c_0; \theta)
    \mathbb{E}_{\pi_p(\cdot;\: c_0)}[
        h(\tau)
    ]
    =
    \mathbb{E}_{\mpt{q$_2$($c_0$)}}[w_2 h(\tau_2)]
    ,
    &&
    c_2,\tau_2,\rho_2,w_2 \eval \mpt{q$_2$(c$_0$)}
    ,
\end{align*}
for any measurable function $h$.
Given evaluations 
$c_1, \tau_1, \rho_1, w_1 \sim \mpt{q$_1$}(c_0)$ and
$c'_2, \tau'_2, \rho'_2, w'_2 \sim \mpt{p}[\tau_1](c_0)$
we can similarly compute a stochastic lower bound  \citep{burda2016importance},
\begin{align*}
    \mathcal{L}
    =
    \mathbb{E}_{\mpt{q$_2$($c_0$)}}
    \left[
        \log w_2
    \right]
    \le
    \log 
    \left(
        \mathbb{E}_{\mpt{q$_2$($c_0$)}}
        \left[
            w_2
        \right]
    \right)
    =
    \log 
    \left(
        Z_p(c_0, \theta)
        \mathbb{E}_{\pi(\cdot; c_0)}
        \left[
            1
        \right]
    \right)
    =   
    \log Z_p(c_0, \theta)
    ,
\end{align*}
using the constant function $h(\tau) = 1$.
The gradient of this bound 
\begin{align*}
    \nabla_\theta
    \mathbb{E}_{\mpt{q$_2$($c_0$)}}
    \left[
        \log
        w_2
    \right]
    &=
    \mathbb{E}_{\mpt{q$_1$($c_0$)}} \left[
        \nabla_\theta \log w_1
        + 
        \nabla_\theta
        \mathbb{E}_{\mpt{p($c_0$)}[\tau_1]}
        \left[
            \log
            \frac{
                \tilde\gamma_{p}(\tilde\tau_2; c_0, \theta)
                p_{q[\tilde\tau_2]}(\tau_1; c_0, \phi)
            }{
                \gamma_{q}(\tau_1; c_0, \phi) 
                p_{p[\tau_1]}(\tilde\tau_2; c_0, \theta)
            } 
        \right]
    \right]
    \\
    &=
    \mathbb{E}_{\mpt{q$_1$($c_0$)}} \left[
        \nabla_\theta
        \mathbb{E}_{\mpt{p($c_0$)}[\tau_1]}
        \left[
            \log
            \frac{
                \tilde\gamma_{p}(\tilde\tau_2; c_0, \theta)
                p_{q[\tilde\tau_2]}(\tau_1; c_0, \phi)
            }{
                p_{p[\tau_1]}(\tilde\tau_2; c_0, \theta)
            } 
        \right]
    \right]
\end{align*}
is a biased estimate of $\nabla_\theta \log Z_p$, where we use Equation~\ref{eq:propose-weight} to replace $w_2$ with the incoming importance weight $w_1$ and the relevant extended-space densities.

If the target program does not introduce additional random variables, i.e. 
$\mathrm{dom}(\tilde \tau_2) \setminus \mathrm{dom}(\tau_1) = \emptyset$, the 
traces produced by the conditioned evaluation $\tilde c_2, \tilde\tau_2, \tilde \rho_2, \tilde w_2 \sim \mpt{p}[\tau_1](c_0)$ do not depend on $\theta$, as all variables are generated from the inference program, and the prior term $p_{p[\tau_1]}(\tilde \tau_2; c_0)=1$.
As a result we can move the gradient operator inside the inner expectation,
\begin{align*}
    \mathbb{E}_{\mpt{q$_1$($c_0$)}} \left[
        \nabla_\theta
        \mathbb{E}_{\mpt{p($c_0$)}[\tau_1]}
        \left[
            \log
            \tilde\gamma_{p}(\tilde\tau_2; c_0, \theta)
            p_{q[\tilde\tau_2]}(\tau_1; c_0)
        \right]
    \right]
    &=
    \mathbb{E}_{\mpt{q$_1$($c_0$)}} \left[
        \mathbb{E}_{\mpt{p($c_0$)}[\tau_1]}
        \left[
            \nabla_\theta
            \log
            \tilde\gamma_{p}(\tilde\tau_2; c_0, \theta)
            p_{q[\tilde\tau_2]}(\tau_1; c_0)
        \right]
    \right]
    \\
    &=
    \mathbb{E}_{\mpt{q$_2$($c_0$)}}
    \left[
        \nabla_\theta
        \log
        \tilde\gamma_{p}(\tilde\tau_2; c_0, \theta)
        p_{q[\tilde\tau_2]}(\tau_1; c_0)
    \right]
    .
\end{align*}
If, additionally, all random variables in the inference program are reused in the target program, i.e. $\mathrm{dom}(\tau_1) \setminus \mathrm{dom}(\tilde\tau_2) = \emptyset$, the prior term $p_{q[\tilde\tau_2]}(\tau_1; c_0)=1$. 
Hence, in the case where the set of random variables in the proposal and target program is the same, i.e. $\mathrm{dom}(\tilde\tau_2) = \mathrm{dom}(\tau_1)$, 
we recover the standard variational inference gradient w.r.t.~the model parameters $\theta$,
\begin{align*}
    \mathbb{E}_{\mpt{q$_2$($c_0$)}}
    \left[
        \nabla_\theta
        \log
        \gamma_{p}(\tilde\tau_2; c_0, \theta)
    \right]
    .
\end{align*}
The gradient with respect to the proposal parameters $\nabla_\phi \mathcal{L}$ can be approximated using likelihood-ratio estimators \citep{wingate2013automated,ranganath2014black}, reparameterized samples \citep{kingma2013auto-encoding,rezende2014stochastic}, or a combination of the two \citep{ritchie2016deep}.
Here we only consider the fully reparameterized case, which allows us to move the gradient operator inside the expectation
\begin{align*}
    \nabla_\phi
    \mathbb{E}_{\mpt{q$_2$($c_0$)}}
    \left[
        \log w_2
    \right]
    &=
    \mathbb{E}_{\mpt{q$_2$($c_0$)}}
    \left[
        \nabla_\phi
            \log w_2
    \right]
    \\
    &=
    \mathbb{E}_{\mpt{q$_2$($c_0$)}}
    \left[
        \frac{\partial \log w_2}{\partial \tilde\tau_2}
        \frac{\partial \tilde\tau_2}{\partial \phi}
        +
        \frac{\partial \log w_2}{\partial \phi}
    \right]
    \\
    &=
    \mathbb{E}_{\mpt{q$_2$($c_0$)}}
    \left[
        \frac{\partial \log w_2}{\partial \tilde\tau_2}
        \frac{\partial \tilde\tau_2}{\partial \phi}
        +
        \frac{\partial}{\partial \phi}
        \log
        \frac{
            p_{q[\tilde\tau_2]}(\tau_1; c_0, \phi)
        }{
            \gamma_{q}(\tau_1; c_0, \phi) 
        } 
    \right]
\end{align*}
In the case where the set of random variables in the proposal and target program is the same, i.e. $\mathrm{dom}(\tilde\tau_2) = \mathrm{dom}(\tau_1)$ and $\mpt{q}_1 = \mpt{f}$ is a primitive program, we can write $w_1=\gamma_f(\tau_1; c_0, \phi)/p_f(\tau_1; c_0, \phi)$, and  
we recover the standard variational inference gradient w.r.t~$\phi$,
\begin{align*}
    &
    \mathbb{E}_{\mpt{q$_2$($c_0$)}}
    \left[
        \frac{\partial}{\partial \tilde\tau_2}
        \log 
        \left(
            \frac{
                    \tilde\gamma_{p}(\tilde\tau_2; c_0, \theta)
                }{
                    p_{f}(\tau_1; c_0, \phi) 
                } 
        \right)
        \frac{\partial \tilde\tau_2}{\partial \phi}
        -
        \frac{\partial}{\partial \phi}
        \log p_{f}(\tau_1; c_0, \phi) 
    \right]
    \\
    &=
    \mathbb{E}_{\mpt{q$_2$($c_0$)}}
    \left[
        \frac{\partial}{\partial \tilde\tau_2}
        \log
        \left(
            \frac{
                    \tilde\gamma_{p}(\tilde\tau_2; c_0, \theta)
                }{
                    p_{f}(\tau_1; c_0, \phi) 
                } 
        \right)
        \frac{\partial \tilde\tau_2}{\partial \phi}
        -
        \frac{\partial}{\partial \phi}
        \log p_{f}(\tau_1; c_0, \phi) 
    \right]
    \\
    &=
    - \nabla_\phi \text{KL}(p_f || \pi_p)
    ,
\end{align*}
effectively minimizing a reverse KL-divergence.
Noticing that the second term is zero in expectation, due to the reinforce trick, we can derive the lower variance gradient
\begin{align*}
    \mathbb{E}_{\mpt{q$_2$($c_0$)}}
    \left[
        \frac{\partial}{\partial \tilde\tau_2}
        \log
        \frac{
                \gamma_{p}(\tilde\tau_2; c_0, \theta)
            }{
                p_{f}(\tau_1; c_0, \phi) 
            } 
        \frac{\partial \tilde\tau_2}{\partial \phi}
        -
        \frac{\partial}{\partial \phi}
        \log p_{f}(\tau_1; c_0, \phi) 
    \right]
    = 
    \mathbb{E}_{\mpt{q$_2$($c_0$)}}
    \left[
        \frac{\partial}{\partial \tilde\tau_2}
        \log
        \frac{
                \gamma_{p}(\tilde\tau_2; c_0, \theta)
            }{
                p_{f}(\tau_1; c_0, \phi) 
            } 
        \frac{\partial \tilde\tau_2}{\partial \phi}
    \right]
    ,
\end{align*}
which can be approximated using reparameterized weights obtained by evaluating $\mpt{q$_2$}$.


\subsection{Reweighted Wake-sleep (RWS) Style Inference.} To implement variational methods inspired by reweighted wake-sleep \citep{hinton1995wake-sleep,bornschein2015reweighted,le2019revisiting}, 
we compute a self-normalized estimate of the gradient
\begin{align}
    \nabla_\theta 
    \log Z_p(c_0, \theta) 
    &= 
    \frac{1}{Z_p(c_0, \theta)}
    \nabla_\theta 
    \int d\tau\ 
    \gamma_p(\tau ; c_0, \theta)
    \\
    &= 
    \frac{1}{Z_p(c_0, \theta)}
    \int d\tau\ 
    \gamma_p(\tau ; c_0, \theta)
    \nabla_\theta 
    \log \gamma_p(\tau ; c_0, \theta)
    \\
    &= 
    \int d\tau\ 
    \pi_p(\tau ; c_0, \theta)
    \nabla_\theta 
    \log \gamma_p(\tau ; c_0, \theta)
    \\
    &=
    \mathbb{E}_{\pi_p(\cdot \:; c_0, \theta)}
    \left[ 
      \nabla_\theta 
      \log \gamma_p(\tau ; c_0, \theta)
    \right].
\end{align}
Notice that here we compute the gradient w.r.t.~ the non-extended density $\gamma_p$, which does not include auxiliary variables and hence density terms which would integrate to one. In practice this allows us to compute lower variance approximation.
Using definition \ref{def:proper_weighting} and traces obtained from the evaluation of our inference program $c_2^l,\tau_2^l,\rho_2^l,w_2^l \eval \mpt{q$_2$(c$_0$)}$ we can derive a costistent self-normalized estimator
\begin{align}
    \mathbb{E}_{\pi_p(\cdot \:; c_0, \theta)}
    \left[ 
      \nabla_\theta 
      \log \gamma_p(\tau ; c_0, \theta)
    \right]
    &=
    Z_p^{-1}(c_0, \theta)
    \mathbb{E}_{\mpt{q$_2$}}
    \left[ 
      w_2
      \nabla_\theta 
      \log \gamma_p(\tau_2 ; c_0, \theta)
    \right]
    \simeq 
    \sum_{l} 
        \frac{
            w_2^l
        }{
            \sum_{l'} w_2^{l'}
        }
        \nabla_\theta 
        \log \gamma_p(\tau_2^l ; c_0, \theta)
    .
\end{align}
We can similarly approximate the gradient of the forward KL divergence with a self-normalized estimator,
\begin{align}  
    -\nabla_\phi 
    \text{KL}(\pi_p || \pi_{q})
    &= 
    \mathbb{E}_{\pi_p(\cdot \:; c_0, \theta)}
    \left[ 
      \nabla_\phi 
      \log \pi_{q}(\tau ; c_0, \phi)
    \right]
    \\
    &= 
    \mathbb{E}_{\pi_p(\cdot \:; c_0, \theta)}
    \left[ 
      \nabla_\phi 
      \log \gamma_{q}(\tau ; c_0, \phi)
    \right]
     -
    \nabla_\phi Z_q(c_0, \phi)
    \\
    &= 
    \mathbb{E}_{\pi_p(\cdot \:; c_0, \theta)}
    \left[ 
      \nabla_\phi 
      \log \gamma_{q}(\tau ; c_0, \phi)
    \right]
    -
    \mathbb{E}_{\pi_q(\cdot \:; c_0, \theta)}
    \left[ 
      \nabla_\phi 
      \log \gamma_{q}(\tau ; c_0, \phi)
    \right]
    \\
    &= 
    Z_p^{-1}(c_0, \theta)
    \mathbb{E}_{\mpt{q$_2$}}
    \left[ 
        w_2
        \nabla_\phi 
        \log \gamma_{q}(\tau_1 ; c_0, \phi)
    \right]
     -
    Z_q^{-1}(c_0, \phi)
    \mathbb{E}_{\mpt{q$_1$}}
    \left[ 
        w_1
        \nabla_\phi 
        \log \gamma_{q}(\tau_1 ; c_0, \phi)
    \right]
    \\
    &\simeq 
     \sum_{l} \left(
      \frac{w_2^l}{\sum_{l'} w_2^{l'}}
      - \frac{w_1^l}{\sum_{l'} w_1^{l'}}
     \right)
    \nabla_\phi \log \gamma_{q}(\tau_1^l ; c_0, \phi).
\end{align}
In the special case where the proposal $\mpt{q}_1 = \mpt{f}$ is a primitive program without observations (i.e. $w^l_1=1$) we have that
\begin{align}
    \mathbb{E}_{\pi_q(\cdot \:; c_0, \theta)}
    \left[ 
      \nabla_\phi 
      \log \gamma_{q}(\tau ; c_0, \phi)
    \right]
    &=
    \mathbb{E}_{\pi_q(\cdot \:; c_0, \theta)}
    \left[ 
      \nabla_\phi 
      \log \pi_q(\tau ; c_0, \phi)
    \right]
    =
    0,
\end{align}
by application of the reinforce trick. 
In this case we drop the second term, which is introducing additional bias, through self-normalization, and variance to the estimator, and recover the standard RWS estimator 
\begin{align}
    \sum_{l}
        \frac{w_2^l}{\sum_{l'} w_2^{l'}}
        \nabla_\phi \log \gamma_{q}(\tau_1^l ; c_0, \phi)
    .
\end{align}

\section{Implementation Details for Experiments}
\subsection{Annealed Variational Inference}
\label{apx:avi}

In the annealing task we implement Annealed Variational Inference \citep{zimmermann2021nested} and learn to sample from a multimodal Gaussian distribution $\gamma_K$, composed of eight equidistantly spaced modes with covariance matrix of $I \cdot 0.5$ on a circle of radius 10.
We define our initial inference program and annealing path to be:

\begin{align*}
    \den{\mpt{q$_0$(c}} = \text{Normal}(\mu=0, \sigma=5)
    && \den{\mpt{q$_k$(c)}}_\gamma = q_1(c)^{1-\beta_k} \gamma(c)_K^{\beta_k}, 
    && \beta_k = \frac{k-1}{K-1},
    && \text{for } k = 1 \ldots K
\end{align*}

Implementations of these programs can be found in Figure \ref{fig:nvi-code}. Additionally, each forward- and reverse- kernel program $\mpt{q$_k$(c)}$ is defined by a neural network:

\begin{align*}
    x = \text{Linear 50. ReLU} (c)
    && \mu_k = \text{Linear 2} (x) + c
    && \text{cov}_k = \text{DiagEmbed . Softplus 2} (x)
\end{align*}

Learned intermediate densities of $\beta_k$ are embedded by a logit function and is extracted by sigmoid function. 

A combinators implementation of nested variational inference is defined:

\begin{center}
\begin{tabular}{c}
\begin{probtorch}
def step(q, intermediate, do_resample):
    (fwd, rev), p = intermediate
    q$'$ = resample(q) if do_resample else q
    return propose(extend(p, rev), compose(q$'$, fwd))

path, kernels = ...
ixs = list(range(len(path)))
do_resamples = map($\lambda$ i $\to$ i == ixs[-1], ixs)
nvir = reduce(step, zip(kernels, path[1:], do_resamples[1:], path[0]))
\end{probtorch}
\end{tabular}
\end{center}

We implement nested variational inference (NVI), nested variational inference with resampling (NVIR), as well as nested variational inference with learned intermediate densities (NVI*), and nested variational inference with resampling and learned intermediate densities (NVI*).
When implementing any NVI algorithm with resampling, we additionally implement nested variational inference with resampling (NVIR*) by applying the \mpt{resample} combinator after all but the final \mpt{proposal} combinator.

We evaluate our model by training each model for 20,000 iterations with a sampling budget of 288 samples, distributed across $K$ intermediate densities.
Metrics of $\log \hat{Z}$ and effective sample size average over 100 batches of 1,000 samples and results are calculated using the mean of 10 training runs using unique, fixed seeds. In the evaluation of NVIR* we do not resample at test time. 

\begin{figure*}[!t]
\begin{minipage}{0.525\linewidth}
\begin{probtorch}
def target(s, x):
  zs, xs = s.sample(Categorical(K), "zs"), []
  for k in K:
    count = sum(zs==k)
    dist = Normal($\mu$(k), $\sigma$(k))
    xs.append(s.sample(dist, str(k), count))
  return s, shuffle(cat(xs))
  
def q_0(s):
  c = s.sample(Normal($\mu$(k), $\sigma$(k)), "q_0")
  return s, c
\end{probtorch}
\end{minipage}
\begin{minipage}{0.475\linewidth}
\begin{probtorch}
$\eta = ...$ # initialize neural network for kernel $k$
def q_k(s, c):
  c$'$ = s.sample(Normal($\eta_{\mu}$(k), $\eta_{\sigma}$(k)), "q_k")
  return s, c$'$

def gamma_k(s, log_gamma, q_0, beta=1.0):
  # sample from the initial proposal
  x = s.sample(q_0, "x")
  # add a heuristic factor
  s.factor(beta * (log_gamma(x) - q0.log_prob(x)))
  return x
\end{probtorch}
\end{minipage}
\caption{\label{fig:nvi-code} models defined for Annealed Variational Inference}
\end{figure*}

\subsection{Amortized Population Gibbs Samplers}
\label{apx:apg_implementation}
Many inference task requires learning a deep generative model. For this purpose, we we evaluate combinators in an unsupervised tracking task. In this task, the data is a corpus of simulated videos that each contain multiple moving objects. Out goal is to learn both the target program (i.e. the generative model) and the inference program using the APG sampler.

Consider a sequence of video frames $x_{1:T}$, which contains $T$ time steps and $D$ different objects. We assume that the $k$th object in the $t$th frame $x_t$ can be represented by some object feature $z^{\mathrm{what}}_{d}$ and a time-dependent position variable $z^{\mathrm{where}}_{d, t}$. The deep generative model takes the form

\begin{align*}    
    z^{\mathrm{what}}_{d} 
    \, &\sim \,
    \text{Normal}(0, \, I),
    &
    z^{\mathrm{where}}_{d, 1} \,&\sim\, \text{Normal}(0,\, I)
    ,
    &
    z^{\mathrm{where}}_{d, t} 
    \, &\sim \,
    \text{Normal}(z^{\mathrm{where}}_{d, t - 1}, \, \sigma^2_0 I)
    ,
    \\
    x_t 
    \,
    &\sim
    \,
    \mathrm{Bernoulli}
    \Big(
        \sigma
        \Big(
            \sum_{d} \mathrm{ST}
            \big(
                g_\theta (z_{d}^{\mathrm{what}})
                , 
                \:
                z^{\mathrm{where}}_{d, t}
            \big)
        \Big)
    \Big)
    ,
    & d &= 1, 2, ..., D,
    &
    t &= 1, 2,.., T
    .
\end{align*}
where $\sigma_0=0.1$ and $z^{\mathrm{what}}_{d} \in \mathbb{R}^{10}$, $z^{\mathrm{where}}_{d, t} \in \mathbb{R}^2$.
To perform inference for this model, APG sampler learns neural proposals to iterate conditional updates to blocks of variables, which consists of one block of object features and T blocks of each time-dependent object position as
\begin{align}
    \{z_{1:D}^{\text{what}}\}
    ,
    \qquad
    \{z_{1:D, \, 1}^{\mathrm{where}}\}
    ,
    \qquad
    \{z_{1:D, \, 2}^{\mathrm{where}}\}
    ,
    \qquad
    \cdots
    ,
    \qquad
    \{z_{1:D, \, T}^{\mathrm{where}}\}
\end{align}
We train the model on 10000 video instances, each containing 10 timesteps and 3 different objects. We train with batch size $5$, sample size $20$, Adam optimizer with $\beta_1=0.9, \beta_2=0.99$ and lr=$2e-4$.

\textbf{Architecture for Generative Model}. We learn a deep generative model of the form
\begin{align}
    p_\theta(x_{1:T} \mid z^{\mathrm{what}}_{1:D}, z^{\mathrm{where}}_{1:T})
    =
    \prod_{t=1}^T
    \text{Bernoulli}
    \Big(
    x_t \:\Big\vert\:
        \sigma
        \Big(
            \sum_{d} \mathrm{ST}
            \big(
                g_\theta (z_{d}^{\mathrm{what}})
                , 
                \:
                z^{\mathrm{where}}_{d, t}
            \big)
        \Big)
    \Big)
\end{align}
Given each object feature $z^{\mathrm{what}}_{d}$, the APG sampler reconstruct a $28\times28$ object image using a MLP decoder, the architecture of which is

\begin{table}[!h]
\centering
\label{arch-bshape-decoder}
\begin{tabular}{l}
    \toprule
     \textbf{Decoder} \hspace{1em} $g_\theta (\cdot)$ \\
    \midrule
    Input $z^{\mathrm{what}}_d \in\mathbb{R}^{10}$ \\
    \hline
    FC 200. ReLU. FC 400. ReLU. FC 784. Sigmoid.\\
    \bottomrule
\end{tabular}
\end{table}

Then we put each reconstructed image $g_\theta(z^{\mathrm{what}}_d)$ onto a $96\times96$ canvas using a spatial transformer ST which takes position variable $z^{\mathrm{where}}_{d, t}$ as input. To ensure a pixel-wise Bernoulli likelihood, we clip on the composition as
\begin{align}
    \text{For each pixel} \, p_i \in \Big( \sum_{d} \mathrm{ST} \big( g_\theta (z_{d}^{\mathrm{what}}), \:z^{\mathrm{where}}_{d, t} \big) \Big),
    \:
    \sigma(p_i)
    = \begin{cases}
        p_i = 0 & \text{if} \, p_i < 0 \\    
        p_i = p_i & \text{if} \, 0 \leq p_i \leq 1 \\   
        p_i = 1 & \text{if} \, p_i > 1 \\
      \end{cases}
\end{align}

\textbf{Architecture for Gibbs Neural Proposals}. The APG sampler in the bouncing object employs neural proposals of the form
\begin{align}
    q_\phi (z^{\mathrm{where}}_{1:D, \, t} \mid x_{t})
    &=
    \prod_{d=1}^D \text{Normal} \Big(z^{\mathrm{where}}_{d, t}  \:\Big\vert\: \tilde{\mu}_{d, t}^\mathrm{where}, \tilde{\sigma}^{\mathrm{where} \, 2}_{d, t} I \Big),
    \qquad \text{for} \: t = 1, 2, \dots, T,
    \\
    q_\phi (z^{\mathrm{where}}_{1:D, \, t} \mid x_{t}, z^{\mathrm{what}}_{1:D})
    &=
    \prod_{d=1}^D \text{Normal} \Big(z^{\mathrm{where}}_{d, t}  \:\Big\vert\: \tilde{\mu}_{d, t}^\mathrm{where}, \tilde{\sigma}^{\mathrm{where} \, 2}_{d, t} I \Big),
    \qquad \text{for} \: t = 1, 2, \dots, T,
    \\
    q_\phi (z^{\mathrm{what}}_{1:D} \mid x_{1:T}, z^{\mathrm{where}}_{1:T})
    &=
    \prod_{d=1}^D \text{Normal} \Big(z^{\mathrm{what}}_{d}  \:\Big\vert\: \tilde{\mu}_d^\mathrm{what}, \tilde{\sigma}^{\mathrm{what} \, 2}_d I \Big)
    .
\end{align}
We train the proposals with instances containing $D=3$ objects and $T=10$ time steps and test them with instances containing up to $D=5$ objects and $T=100$ time steps. We use the tilde symbol $\,\tilde{}\,$ to denote the parameters of the conditional neural proposals (i.e. approximate Gibbs proposals).

The APG sampler uses these proposals to iterate over the $T+1$ blocks 
\begin{align*}
    \{z_{1:D}^{\mathrm{what}}\}, 
    \qquad
    \{z_{1:D, \, 1}^{\mathrm{where}}\}, 
    \qquad
    \{z_{1:D, \, 2}^{\mathrm{where}}\}, 
    \qquad
    \dots, 
    \qquad
    \{z_{1:D, \, T}^{\mathrm{where}}\}.
\end{align*}

For the position features, the proposal $q_\phi (z^{\mathrm{where}}_{1:D, \, t} \mid x_{t})$ and proposal $q_\phi (z^{\mathrm{where}}_{1:D, \, t} \mid x_{t}, z^{\mathrm{what}}_{1:D})$ share the same network, but contain different pre-steps where we compute the input of that network. The initial proposal $q_\phi (z^{\mathrm{where}}_{1:D, \, t} \mid x_{t})$ will convolve the frame $x_t$ with the mean image of the object dataset; The conditional proposal $q_\phi (z^{\mathrm{where}}_{1:D, \, t} \mid x_{t}, z^{\mathrm{what}}_{1:D})$ will convolve the frame $x_t$ with each reconstructed object image $g_\theta(z^\mathrm{what}_d)$. We perform convolution sequentially by looping over all objects $d=1, 2, ..., D$. Here is pseudocode of both pre-steps:

\begin{algorithm}[!h]
  \caption{Convolution Processing for  $q_\phi (z^{\mathrm{where}}_{1:D, \, t} \mid x_{t})$}
\begin{algorithmic}[1]
  \State \textbf{Input} frame $x_t\in\mathbb{R}^{9216}$, mean image of object dataset $mm \in \mathbb{R}^{784}$
  \For {$d = 1$ \textbf{to} $D$}
    \State $x_{d, t}^{\text{conv}} \xleftarrow{} \text{Conv2d}(x_t)$ with kernel $mm$, stride = 1, no padding.
    \EndFor
  \State \textbf{Output} Convolved features $\{x_{d, t}^{\text{conv}} \in\mathbb{R}^{4761}\}_{d=1}^D$
\end{algorithmic}
\end{algorithm}

\begin{algorithm}[!h]
  \caption{Convolution Processing for  $q_\phi (z^{\mathrm{where}}_{1:D, \, t} \mid x_{t}, z^{\mathrm{what}}_{1:D})$}
\begin{algorithmic}[1]
  \State \textbf{Input} frame $x_t\in\mathbb{R}^{9216}$, reconstructed object objects $\{g_\theta(z^\mathrm{what}_d) \in \mathbb{R}^{784}\}_{d=1}^D$
  \For {$d = 1$ \textbf{to} $D$}
    \State $x_{d, t}^{\text{conv}} \xleftarrow{} \text{Conv2d}(x_t)$ with kernel $g_\theta(z^\mathrm{what}_d)$, stride = 1, no padding.
    \EndFor
  \State \textbf{Output} Convolved features $\{x_{d, t}^{\text{conv}} \in\mathbb{R}^{4761}\}_{d=1}^D$
\end{algorithmic}
\end{algorithm}

We employ a MLP encoder $f_\phi^\textsc{l}(\cdot)$ that takes the convolved features as input and predict the variational parameters for positions $\{z^\mathrm{where}_{d, t}\}_{d=1}^D$ at step $t$, i.e. vector-valued mean $\tilde{\mu}^\mathrm{where}_{d, t}$ and logarithm of the diagonal covariance $\log \tilde{\sigma}^{\mathrm{where} \, 2}_{d, t}$ as
\begin{align}
    &   
    \tilde{\mu}^\mathrm{where}_{d, t}, \log \tilde{\sigma}^{\mathrm{where} \, 2}_{d, t} \xleftarrow{} f_\phi^\textsc{l}(x_{d, t}^{\text{conv}})
    ,
    &
    d = 1, 2, \dots, D.
\end{align}
The architecture of the MLP encoder $f_\phi^\textsc{l}(\cdot)$ is
\begin{table}[!t]
\centering
\begin{tabular}{l}
 \toprule
\textbf{Encoder} \hspace{1em} $f_\phi^\textsc{l}(\cdot)$ \\
\midrule
Input $x_{d, t}^{\text{conv}} \in \mathbb{R}^{4761}$ \\
\hline
FC 200. ReLU. FC $2\times100$. ReLU. FC $2\times2$. \\
\bottomrule
\end{tabular}
\label{arch-bshape-enc-where}    
\end{table}

For the object features, the APG sampler performs conditional updates in the sense that we crop each frame $x_t$ into a $28\times28$ subframe according to $z^\mathrm{where}_{d, t}$ using the spatial transformer ST as
\begin{align}
    &x^\text{crop}_{d, t} \xleftarrow{} \text{ST}\big(x_t, z^{\mathrm{where}}_{d, t}\big), &d = 1, 2, \dots, D, \quad t = 1, 2, \dots, T
    .
\end{align}
we employ a MLP encoder $T_\phi^\textsc{g}(\cdot)$ that takes the cropped subframes as input, and predicts frame-wise neural sufficient statistics, which we will sum up over all the time steps.

Then we employ another network $f_\phi^\textsc{g}(\cdot)$ that takes the sums as input, and predict the variational parameters for object features $\{z^\mathrm{what}_{d}\}_{d=1}^D$, i.e. the vector-valued means $\{\tilde{\mu}_d^\mathrm{what}\}_{d=1}^D$ and the logarithms of the diagonal covariances $\{\log \tilde{\sigma}_d^{\mathrm{what} \, 2}\}_{d=1}^D$.  The architecture of this network is
\begin{table}[!h]
    \centering
    \begin{tabular}{l}
    \toprule
    \textbf{Encoder} \hspace{1em} $f_\phi^\textsc{g} (\cdot)$ \\
    \midrule
    Input $x^\text{crop}_{d, t}\in\mathbb{R}^{784}$ \\
    \hline
    FC 400. ReLU. FC 200. ReLU. $\xrightarrow{}$ Neural Sufficient Statistics $T_\phi^\textsc{g}(x^\text{crop}_{d, t}) \in \mathbb{R}^{200}$ \\
    \hline
    Intermediate Input $\sum_{t=1}^T T_\phi^\textsc{g}(x^\text{crop}_{d, t})$ $\xrightarrow{}$ FC $2\times 10$. \\
    \bottomrule
    \end{tabular}
\end{table}

\printbibliography
\end{refsection}
\end{document}